\documentclass{article}

\usepackage[accepted]{icml2022}

\usepackage[utf8]{inputenc} 
\usepackage[T1]{fontenc}    
\usepackage{microtype}      

\usepackage{xcolor}
\colorlet{citeblue}{blue!50!black}
\colorlet{varreductioncolor}{red!60!black}
\colorlet{precondcgcolor}{green!40!black}

\usepackage{enumitem}
\setitemize{topsep=0pt,parsep=0pt,partopsep=0pt} 

\usepackage{footnote}
\makesavenoteenv{tabular}
\makesavenoteenv{table} 

\usepackage{amsmath}
\usepackage{amssymb}
\usepackage{amsfonts}       
\usepackage{nicefrac}       
\usepackage{mathtools}
\usepackage{relsize}
\usepackage{bm}

\usepackage{etoolbox}
\robustify\bfseries
\usepackage{pgfplots}
\usepgfplotslibrary{groupplots}
\usepackage{tikz}
\usetikzlibrary{tikzmark}
\usetikzlibrary{calc}

\usetikzlibrary{external}
\tikzset{external/force remake=false}
\tikzsetexternalprefix{figures/external/}

\usepackage{booktabs}
\usepackage{rotating}
\usepackage{siunitx}
\usepackage{multirow}

\usepackage{graphicx}
\usepackage{pgfplots}
\pgfplotsset{compat=1.15}
\usepackage[tableposition=top]{caption}
\usepackage[labelformat=simple]{subcaption}
\usepackage{adjustbox}

\usepackage{algorithm}
\usepackage[noend]{algpseudocode}
\algrenewcommand{\algorithmiccomment}[1]{\hfill {\small \textcolor{darkgray}{$\mathsmaller \vartriangleright$ #1}}}  
\algrenewcommand\algorithmicindent{1.5em}   
\algrenewcommand\alglinenumber[1]{\small {\textcolor{darkgray}{#1}}} 
\algrenewcommand{\algorithmicrequire}{\textbf{Input:}}

\makeatletter
\expandafter\patchcmd\csname\string\algorithmic\endcsname{\itemsep\z@}{\itemsep=0.25ex}{}{}
\newcommand\fs@booktabsruled{%
  \def\@fs@cfont{\bfseries\strut}\let\@fs@capt\floatc@ruled
  \def\@fs@pre{\hrule height\heavyrulewidth depth0pt \kern\belowrulesep}%
  \def\@fs@mid{\kern\aboverulesep\hrule height\lightrulewidth\kern\belowrulesep}%
  \def\@fs@post{\kern\aboverulesep\hrule height\heavyrulewidth\relax}%
  \let\@fs@iftopcapt\iftrue
}
\makeatother
\floatstyle{booktabsruled}
\restylefloat{algorithm}
\usepackage{amsthm}
\usepackage{thmtools}
\usepackage{thm-restate}
\newtheoremstyle{theorem-style}
{\topsep} 
{\topsep} 
{\itshape} 
{} 
{\bfseries} 
{} 
{\newline} 
{} 
\theoremstyle{theorem-style}
\newtheorem{theorem}{Theorem}
\newtheorem{proposition}{Proposition}
\newtheorem{corollary}{Corollary}
\newtheorem{lemma}{Lemma}

\newtheoremstyle{definition-style}
{\topsep} 
{\topsep} 
{} 
{} 
{\bfseries} 
{} 
{\newline} 
{} 
\theoremstyle{definition-style}
\newtheorem{definition}{Definition}
\newtheorem{remark}{Remark}

\usepackage{titletoc}
\usepackage{tocloft}
\usepackage[prependcaption,textsize=small,color=gray!40]{todonotes} 

\newif\ifcomments
\commentstrue
\ifcomments\newcommand{\comments}[1]{#1}\else\newcommand{\comments}[1]{}\fi

\newcommand{\beginsupplementary}{%

  \renewcommand{\thesection}{S\arabic{section}}
  \renewcommand{\thesubsection}{\thesection.\arabic{subsection}}
  \renewcommand{\theHsection}{S\arabic{section}} 

  \setcounter{table}{0}
  \renewcommand{\thetable}{S\arabic{table}}
  \setcounter{figure}{0}
  \renewcommand{\thefigure}{S\arabic{figure}}
  \setcounter{section}{0}
  \renewcommand{\theequation}{S\arabic{equation}} 
  \renewcommand{\thetheorem}{S\arabic{theorem}} 
  \renewcommand{\thedefinition}{S\arabic{definition}} 
  \renewcommand{\thecorollary}{S\arabic{corollary}} 
  \renewcommand{\theproposition}{S\arabic{proposition}} 
  \renewcommand{\theremark}{S\arabic{remark}} 
  \renewcommand{\thelemma}{S\arabic{lemma}} 
  \renewcommand{\theexample}{S\arabic{example}} 
}

\usepackage{url}
\usepackage{hyperref}
\hypersetup{
  colorlinks,
  linkcolor={black},
  citecolor=citeblue,
  urlcolor=citeblue,
  breaklinks=true,   
  linktoc=all,
}
\usepackage[nameinlink]{cleveref}
\usepackage[numbered]{bookmark} 

\usepackage{amsmath}
\usepackage{amsfonts}
\usepackage{amssymb}
\usepackage{bm}
\usepackage{physics}

\newcommand{\shat}[1]{\vphantom{#1}\smash[t]{\hat{#1}}}

\newcommand{\N}{\mathbb{N}}

\newcommand{\R}{\mathbb{R}}

\newcommand{\Rn}{\mathbb{R}^n}

\newcommand{\Rnn}{\mathbb{R}^{n\times n}}

\renewcommand{\top}{{\intercal}}

\DeclareSymbolFont{stmry}{U}{stmry}{m}{n}
\DeclareMathSymbol\obar\mathrel{stmry}{"3A}
\DeclareMathSymbol\otimes\mathrel{stmry}{"0F}
\DeclareMathSymbol\ominus\mathrel{stmry}{"17}
\makeatletter
\newcommand{\superimpose}[2]{
  {\ooalign{$#1\@firstoftwo#2$\cr\hfil$#1\@secondoftwo#2$\hfil\cr}}}
\makeatother

\makeatother

\renewcommand{\Pr}{\mathbb{P}}
\DeclareMathOperator{\Exp}{\mathbb{E}}

\DeclareMathOperator{\Cov}{\mathrm{Cov}}

\newcommand{\Normal}{\mathcal{N}}
\newcommand{\GP}{\mathcal{GP}}

\newcommand{\logmarglik}{\mathcal{L}}

\newcommand{\bigO}{\mathcal{O}}

\def\vb{{\bm{b}}}

\def\ve{{\bm{e}}}
\def\vf{{\bm{\mathrm{f}}}}

\def\vu{{\bm{u}}}
\def\vv{{\bm{v}}}
\def\vw{{\bm{w}}}
\def\vx{{\bm{x}}}
\def\vy{{\bm{y}}}
\def\vz{{\bm{z}}}

\def\vmu{{\bm{\mu}}}
\def\vtheta{{\bm{\theta}}}

\def\vlambda{{\bm{\lambda}}}

\def\vzero{{\bm{0}}}
\def\vone{{\bm{1}}}

\def\evlambda{{\lambda}}

\def\evtheta{{\theta}}

\def\mA{{\bm{A}}}
\def\mB{{\bm{B}}}
\def\mC{{\bm{C}}}

\def\mH{{\bm{H}}}
\def\mI{{\bm{I}}}

\def\mK{{\bm{K}}}
\def\mL{{\bm{L}}}
\def\mM{{\bm{M}}}

\def\mP{{\bm{P}}}
\def\mQ{{\bm{Q}}}

\def\mT{{\bm{T}}}

\def\mV{{\bm{V}}}
\def\mW{{\bm{W}}}
\def\mX{{\bm{X}}}

\def\mZ{{\bm{Z}}}

\def\mDelta{{\bm{\Delta}}}

\def\mLambda{{\bm{\Lambda}}}
\def\mSigma{{\bm{\Sigma}}}

\def\mZero{{\bm{0}}}

\DeclareMathAlphabet{\mathsfit}{\encodingdefault}{\sfdefault}{m}{sl}
\SetMathAlphabet{\mathsfit}{bold}{\encodingdefault}{\sfdefault}{bx}{n}

\def\idxCG{{m}}
\def\idxLanczos{m}
\def\idxrvs{\ell}
\def\idxtxtSTE{\mathrm{STE}}
\def\idxtxtSLQ{\mathrm{SLQ}}
\def\idxtxtSCG{\mathrm{SCG}}
\def\idxtxtCG{\mathrm{CG}}
\def\idxtxtLanczosSTE{\mathrm{Lanczos}}
\def\idxtxtCGSTE{\mathrm{CG}'}

\def\trestimate{\tau}
\def\trSTE{\trestimate_\idxrvs^{\idxtxtSTE}}
\def\trSLQ{\trestimate_{\idxrvs, \idxLanczos}^{\idxtxtSLQ}}
\def\trSCG{\trestimate_{\idxrvs, \idxLanczos}^{\idxtxtSCG}}
\def\trvarredux{\trestimate_{*}}
\def\trvarreduxlogdet{\trvarredux^{\log}}
\def\trvarreduxinvderiv{\trvarredux^{\mathrm{inv}\partial}}
\def\trprecondlogdet{\trestimate_{\shat{\mP}}^{\log}}
\def\trprecondinvderiv{\trestimate_{\shat{\mP}}^{\mathrm{inv}\partial}}

\newcommand{\papertitle}{Preconditioning for Scalable Gaussian Process Hyperparameter Optimization}
\icmltitlerunning{Preconditioning for Scalable GP Hyperparameter Optimization}

\begin{document}

\twocolumn[
	\icmltitle{\papertitle}

	\begin{icmlauthorlist}
		\icmlauthor{Jonathan Wenger}{unitue,mpiis,columbia}
		\icmlauthor{Geoff Pleiss}{columbia}
		\icmlauthor{Philipp Hennig}{unitue,mpiis}
		\icmlauthor{John P. Cunningham}{columbia}
		\icmlauthor{Jacob R. Gardner}{upenn}
	\end{icmlauthorlist}

	\icmlaffiliation{unitue}{University of T\" ubingen}
	\icmlaffiliation{mpiis}{Max Planck Institute for Intelligent Systems, T\" ubingen}
	\icmlaffiliation{columbia}{Columbia University}
	\icmlaffiliation{upenn}{University of Pennsylvania}

	\icmlcorrespondingauthor{Jonathan Wenger}{jonathan.wenger@uni-tuebingen.de}

	\icmlkeywords{Gaussian processes, preconditioning, numerical linear algebra}

	\vskip 0.3in
]

\printAffiliationsAndNotice{}

\begin{abstract}
	Gaussian process hyperparameter optimization requires linear solves
	with, and $\log$-determinants of, large kernel matrices.
	Iterative numerical techniques are becoming popular to scale to larger datasets,
	relying on the conjugate gradient method (CG) for the linear solves
	and stochastic trace estimation for the $\log$-determinant.
	This work introduces new algorithmic and theoretical insights for preconditioning these
	computations.
	While preconditioning is well understood in the context of CG,
	we demonstrate that it can also accelerate convergence and reduce variance of the
	estimates for the $\log$-determinant and its derivative.
	We prove general probabilistic
	error bounds for the preconditioned computation of the $\log$-determinant,
	$\log$-marginal likelihood and its derivatives. Additionally, we derive specific
	rates for a range of kernel-preconditioner combinations, showing that up to
	exponential convergence can be achieved. Our theoretical results enable provably
	efficient optimization of kernel hyperparameters, which we validate empirically on
	large-scale benchmark problems. There our approach accelerates training by up to an order of
	magnitude.
\end{abstract}

\section{Introduction}

Gaussian processes (GPs) are a theoretically well-founded and powerful
probabilistic model \cite{Rasmussen2006}. However, conditioning a GP on data is often
computationally prohibitive for large datasets. This problem is amplified when
optimizing kernel hyperparameters. Gradient-based optimization requires
repeated evaluation of the \(\log\)-marginal likelihood \(\mathcal{L}\) and its
derivatives. These computations both have cubic complexity in the size \(n\) of the data.

\begin{figure}
	\hspace{4.5em}
	\includegraphics[width=0.75\columnwidth]{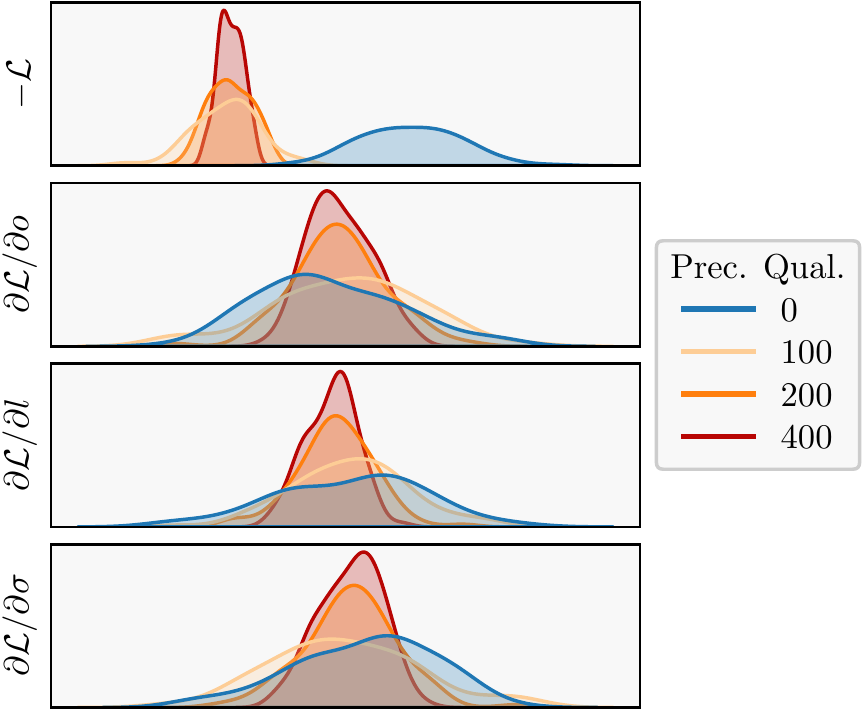}
	\caption{\emph{Preconditioning reduces not only bias but also variance in stochastic approximations to the
			\(\log\)-marginal likelihood \(\logmarglik\) and its derivatives.} GP hyperparameter optimization for large datasets requires cheap
		estimates of \(\logmarglik\) and its gradient. Preconditioning makes these
		more precise and less noisy as is shown here for increasing preconditioner quality on the
		``Elevators'' dataset using a Mat{\'e}rn\((\frac{3}{2})\) kernel.}
	\label{fig:precond-gp-hyperopt}
\end{figure}

\paragraph{GP Inference via Matrix-Vector Multiplication}
Recently, Krylov methods \cite{Golub2013}, based on iterative matrix-vector
multiplication
with the kernel matrix, have become popular for GP inference
\cite{Murray2009,Anitescu2012,Cutajar2016,Ubaru2017,Dong2017,Gardner2018,Wang2019}.
Methods primarily relying on matrix-vector products are advantageous. They can leverage structure in the kernel matrix \cite{Nocedal2006}, and importantly, they make effective use of modern
hardware and parallelization \cite{Gardner2018,Wang2019,Charlier2021}. When optimizing GP hyperparameters one needs to repeatedly solve linear systems with, and compute \(\log\)-determinants of, the kernel matrix.
Both can be done using Krylov methods. The linear systems are solved via the conjugate gradient method (CG)
\cite{Hestenes1952}, which reduces the cost of kernel matrix solves from
\(\bigO(n^3)\) to \(\bigO(n^2\idxCG)\) for \(\idxCG\) iterations. The \(\log\)-determinant can be approximated via stochastic trace estimation (STE) \cite{Hutchinson1989} combined with another Krylov method, the Lanczos algorithm \cite{Lanczos1950} (see \cite{Ubaru2017}). Its derivative may also be estimated via STE combined with CG \cite{Gardner2018}.

\paragraph{Challenges with this Approach}
Despite the advantages of combining Krylov methods with stochastic trace
estimation, there are considerable challenges in practice. These essentially reduce to bias and variance of the numerical approximations.
First, the convergence of CG depends on the conditioning of the kernel matrix, which can grow rapidly with \(n\) (e.g. for the RBF kernel).
Many iterations may be needed to achieve a desired error, and stopping the solver
early can result in biased solutions \citep{Potapczynski2021}.
However, if a preconditioner -- i.e. an approximation of the kernel matrix -- is available, convergence can be accelerated substantially \cite{Golub2013}.
Second, stochastic approximations of the \(\log\)-determinant and its derivative introduce variance into hyperparameter optimization.
While the estimates are unbiased (assuming sufficient Krylov iterations), variance can significantly slow down optimization. 
Reducing variance either requires further approximation at the cost of more bias
\citep{Artemev2021},
or a larger number of samples \(\idxrvs\) which only reduces error at a rate of
\(\bigO(\idxrvs^{-\frac{1}{2}})\) \cite{Avron2011}.
Now, while preconditioning is known to accelerate CG, it has not yet been explored for stochastic trace estimation in this context. 

\paragraph{Contributions}
We demonstrate that, with only a small algorithmic modification, preconditioning can be exploited for highly efficient \(\log\)-determinant estimation, and in turn GP hyperparameter optimization. We show that 
\begin{enumerate}[noitemsep,topsep=0pt,label=(\alph*)]
	\item \emph{preconditioning reduces variance} -- or equivalently accelerates convergence -- of the stochastic estimate of
	      the \(\log\)-determinant and its derivative (\Cref{thm:variance-reduced-trace-estimate}).
\end{enumerate}
We leverage this result, illustrated in \Cref{fig:precond-gp-hyperopt}, to prove
\begin{enumerate}[noitemsep,topsep=0pt,label=(\alph*),start=2]
	\item \emph{stronger theoretical guarantees} for the computation of the \(\log\)-determinant
	      (\Cref{thm:convergence-logdet,thm:convergence-logdet-backward}) and \(\log\)-marginal likelihood (\Cref{thm:convergence-log-marginal-likelihood}) than
	      previously known \cite{Ubaru2017,Gardner2018} and a \emph{novel error bound for the derivative} (\Cref{thm:convergence-derivative}).
\end{enumerate}
To make these general results concrete, we derive
\begin{enumerate}[noitemsep,topsep=0pt,label=(\alph*),start=3]
	\item \emph{specific rates for important combinations of kernels and preconditioners} (\Cref{tab:residual-decay-kernels-preconditioners}), making preconditioner choice for GP inference rigorous rather than heuristic.
\end{enumerate}
Finally, using our approach, we empirically observe
\begin{enumerate}[noitemsep,topsep=0pt,label=(\alph*),start=4]
	\item \emph{up to twelvefold speedup in training} of GP regression models applied to large-scale benchmark problems with up to
	      \(n \approx 325,\!000\) datapoints.
\end{enumerate}

\section{Background}
We want to infer a map \(h :\mathcal{X} \to \mathcal{Y}\)
from an input space \(\mathcal{X} \subset \R^d\) to an output space
\(\mathcal{Y} \subset \R\), given a dataset \(\mX \in \R^{n \times d}\) of
\(n\) training inputs \(\vx_i \in \R^d\) and outputs \(\vy \in \Rn\).

\subsection{Gaussian Processes}

A stochastic process \(f \sim \GP(\mu, k)\) with mean function
\(\mu\) and kernel \(k\) is called a \emph{Gaussian process} if \(\vf
= (f(\vx_1), \dots, f(\vx_n))^\top \sim
\Normal(\vmu, \mK)\) is jointly Gaussian with mean \(\vmu_i
= \mu(\vx_i)\) and
covariance \(\mK_{ij} =
k(\vx_i,
\vx_j)\).
Assuming \(\vy \mid \vf \sim \Normal(\vf, \sigma^2 \mI)\), the posterior
distribution
for test inputs \(\vx_\star\) is also Gaussian with
\begin{align*}
	\Exp[\vf_\star] & = \mu(\vx_\star) + k(\vx_\star, \mX)\shat{\mK}^{-1} (\vy - \vmu),                                       \\
	\Cov(\vf_\star) & = k(\vx_\star, \vx_\star) - k(\vx_\star, \mX)\shat{\mK}^{-1}k(\mX, \vx_\star),
\end{align*}
where \(\shat{\mK} = \mK + \sigma^2 \mI\). Without loss of generality we assume \(\mu = 0\) from now on.

\paragraph{Hyperparameter Optimization}
The computational bottleneck when optimizing kernel hyperparameters \(\vtheta\) is the repeated
evaluation of the \emph{\(\log\)-marginal likelihood}
\begin{equation}
	\label{eqn:likelihood}
	\begin{aligned}
		\logmarglik(\vtheta) & =\log p(\vy \mid \mX, \vtheta)                                                                                                   \\
		                     & =-\frac{1}{2} \big(\vy^\top \shat{\mK}^{-1} \vy + \underbrace{\log\det(\shat{\mK})}_{=\tr(\log(\shat{\mK}))} + n \log(2\pi)\big)
	\end{aligned}
\end{equation}
and its \emph{derivative} with respect to the hyperparameters
\begin{equation}
	{\textstyle \pdv{}{\evtheta}\logmarglik(\vtheta)} = \frac{1}{2}\vy^\top \shat{\mK}^{-1}
	{\textstyle \pdv{\shat{\mK}}{\evtheta} } \shat{\mK}^{-1}\vy - \frac{1}{2}
	\tr(\shat{\mK}^{-1} {\textstyle \pdv{\shat{\mK}}{\evtheta}}).
	\label{eqn:derivative_likelihood}
\end{equation}
Computing \eqref{eqn:likelihood} and \eqref{eqn:derivative_likelihood} via a Cholesky
decomposition has complexity \(\mathcal{O}(n^3)\) which is prohibitive for large \(n\). In response, many methods for approximate GP inference were developed \citep{Titsias2009,Hensman2013,Wilson2015,Wilson2015a,Pleiss2018}. 

In contrast, we aim for tractable \emph{numerically exact} GP inference in the large-scale setting \cite{Gardner2018,Wang2019}. To achieve this, we focus on efficient computation of the \(\log\)-determinant in \eqref{eqn:likelihood} and its derivative in \eqref{eqn:derivative_likelihood} (see \Cref{sec:precond-log-det-estimation}). This allows us to theoretically (\Cref{sec:precond-gp-hyperparam-opt}) and
empirically (\Cref{sec:experiments}) accelerate GP
hyperparameter optimization.

\subsection{Numerical Toolbox for Inference}

We will use the following
established numerical techniques.

\paragraph{Stochastic Trace Estimation (STE)}
The trace \(\tr(\mA)\) of a matrix can be approximated by drawing \(\idxrvs\)
independent random
vectors
\(\vz_i\) with \(\Exp[\vz_i] = \vzero\) and \(\Cov(\sqrt{n}\vz_i) =\mI\)
and computing
\emph{Hutchinson's estimator}
\cite{Hutchinson1989}
\begin{equation}
	\label{eqn:hutchinsons_estimator}
	\textstyle \trSTE(\mA) = \frac{n}{\idxrvs}\sum_{i=1}^\idxrvs
	\vz_i^\top \mA
	\vz_i \approx \tr(\mA).
\end{equation}
Here, we additionally assume the random vectors are normalized\footnote{Normalization is necessary for Lanczos quadrature \citep[see][Chap.~7.2]{Golub2009}.} \(\vz_i =
\tilde{\vz}_i / \norm{\tilde{\vz}_i}_2\) and that \(\sqrt{n}(\vz_1, \dots, \vz_\idxrvs)^\top \in
\R^{\idxrvs n}\) satisfies the
convex concentration property\footnote{Concentration enables us to prove probabilistic error bounds.} (see \Cref{def:convex-concentration-property}). These
assumptions are fulfilled by Rademacher-distributed random
vectors \(\tilde{\vz}_i\) with entries \(\{+1, -1\}\).\footnote{We conjecture they also hold for vectors	\(\tilde{\vz}_i
	\sim \Normal(\vzero, \mI)\).}
Evaluating the \(\log\)-marginal likelihood \eqref{eqn:likelihood} requires computing \(\tr(\log(\shat{\mK}))\). To use Hutchinson's estimator, we need to
efficiently compute \(\idxrvs\) quadratic terms \(\{\vz_i^\top \log(\shat{\mK}) \vz_i\}_{i=1}^\idxrvs\).

\paragraph{Stochastic Lanczos Quadrature (SLQ)}
Given a matrix function \(f\), one can approximate bilinear
forms \(\vz^\top f(\shat{\mK}) \vz\) using quadrature \citep[Chap.~7]{Golub2009}. The
nodes and weights of the quadrature rule can be computed
efficiently via \(\idxLanczos\) iterations of the \emph{Lanczos algorithm} \cite{Lanczos1950}  (or equivalently via CG \cite{Gardner2018}).
The combination with Hutchinson's estimator
\begin{equation}
	\trSLQ(f(\shat{\mK})) \approx \trSTE(f(\shat{\mK})) \approx \tr(f(\shat{\mK})),
\end{equation}
is called \emph{stochastic Lanczos quadrature} \cite{Ubaru2017}.

To compute the linear solves \(\vv \mapsto \shat{\mK}^{-1}\vv\) with the kernel matrix in \eqref{eqn:likelihood} and \eqref{eqn:derivative_likelihood}, we use the \emph{conjugate gradient method}.

\paragraph{Conjugate Gradient Method (CG)}
CG \cite{Hestenes1952} is an iterative method for
solving linear systems with symmetric positive definite matrix. It is particularly suited for
large systems since it is matrix-free, and relies primarily on matrix-vector multiplication with \(\shat{\mK}\).

\paragraph{Preconditioning}
It is well-known that CG can be accelerated via a symmetric positive definite
preconditioner \(\shat{\mP} \approx \shat{\mK}\), by solving an
equivalent linear system with matrix
\(\shat{\mP}^{-\frac{1}{2}}
\shat{\mK} \shat{\mP}^{-\frac{\top}{2}} \approx \mI\) \cite{Trefethen1997}.
CG's convergence is then determined by the condition number
\begin{equation}
	\textstyle  \kappa \coloneqq \kappa(\shat{\mP}^{-\frac{1}{2}}
	\shat{\mK} \shat{\mP}^{-\frac{\top}{2}}) \ll \kappa(\shat{\mK}) =
	\frac{\abs{\lambda_{\max}(\shat{\mK})}}{\abs{\lambda_{\min}(\shat{\mK})}}.
\end{equation}
Suppose the approximation quality of a sequence
of preconditioners \(\{\shat{\mP}_\idxrvs\}_\idxrvs\) indexed by \(\idxrvs\) is given by\footnote{The use of \(\idxrvs\) for the number of random vectors and the preconditioner sequence is deliberate. Setting them to the same value enables variance reduction as we prove in \Cref{thm:variance-reduced-trace-estimate}.}
\begin{equation}
	\label{eqn:preconditioner-quality}
	\lVert{\shat{\mK} - \shat{\mP}_\idxrvs}\rVert_F \leq \bigO(g(\idxrvs)) \lVert \shat{\mK} \rVert_F.
\end{equation}
If \(g(\idxrvs) \to 0\) quickly, a small amount
of precomputation can significantly accelerate CG, since by \Cref{lem:condition-number-preconditioner-quality}
\begin{equation}
	\label{eqn:condition-preconditioner-quality}
	\kappa \leq (1 + \bigO(g(\idxrvs))\norm{\shat{\mK}}_F)^2.
\end{equation}
Preconditioners must be cheap to
obtain and allow efficient linear solves \(\vv \mapsto \shat{\mP}^{-1}\vv\).\footnote{While
CG (and Lanczos) assume a symmetric
pos. definite matrix, both can be implemented using only
$\shat{\mP}^{-1}$, not \(\shat{\mP}^{-\frac{1}{2}}
\shat{\mK}
\shat{\mP}^{-\frac{\top}{2}}\).} 
As an example, \emph{diagonal-plus-low-rank} preconditioners \(\shat{\mP}_\idxrvs = \sigma^2\mI +
\mL_\idxrvs \mL_\idxrvs^\top\),
with \(\mL_\idxrvs \in \R^{n \times \idxrvs}\), admit linear solves in \(\bigO(n \idxrvs^2)\) via
the matrix inversion lemma.

\section{Log-Determinant Estimation}
\label{sec:precond-log-det-estimation}

Our goal is to compute \(\log \det(\shat{\mK}) = \tr (\log(\shat{\mK}))\) and its
derivative via
matrix-vector multiplication. As described, we can use stochastic trace estimation to do
so. Now assume we additionally have access to a preconditioner \(\shat{\mP} \approx
\shat{\mK}\).
As we will show, we can then not just accelerate the convergence of CG, but \emph{also} more efficiently compute the forward and backward pass for the \(\log\)-determinant.

By the properties of the matrix logarithm we can decompose the \(\log\)-determinant into a
\emph{deterministic approximation} based on the preconditioner and a \emph{residual trace} computed via
stochastic trace estimation.\footnote{Similar approaches have been suggested by \citet{Adams2018,Meyer2021}. Our
	work
	is notably different in that it a) uses preconditioning, b) also considers the backward pass and c) gives stronger theoretical guarantees.} It holds by \Cref{lem:logdet-decomposition}, that
\begin{equation}
	\label{eqn:logdet-decomposition}
	\log\det(\shat{\mK}) = \log\det(\shat{\mP}_\idxrvs) + \mathrm{tr}(\underbrace{\log(\shat{\mK}) - \log(\shat{\mP}_\idxrvs)}_{=\mDelta_{\log}})
	\vspace{-0.5em}
\end{equation}
where $\tr(\mDelta_{\log})=\mathrm{tr}\big(\log\big(\shat{\mP}_\idxrvs^{-\frac{1}{2}}
\shat{\mK} \shat{\mP}_\idxrvs^{-\frac{\top}{2}}\big)\big)$ and we assume $\log \det(\shat{\mP}_\idxrvs)$ is efficient to compute.
\Cref{eqn:logdet-decomposition} has two crucial benefits we can exploit. First and foremost, the faster
\(\shat{\mP}_\idxrvs \to \shat{\mK}\), i.e. \(g(\idxrvs) \to 0\), the less the stochastic approximation of \(\tr(\mDelta_{\log})\) affects the
estimate. Since its contribution to the overall error decreases the better \(\log\det(\shat{\mP}_\idxrvs)\) approximates \(\log\det(\shat{\mK}_\idxrvs)\), significantly fewer random vectors are needed to achieve a desired error with high probability. Second, we can now run Lanczos on the preconditioned matrix accelerating its convergence. As we will show later, one can also exploit \eqref{eqn:logdet-decomposition} for the backward pass.

\subsection{Variance-reduced Stochastic Trace Estimation}

This intuitive argument for the \(\log\)-determinant also holds generally, assuming a similar decomposition exists.

\begin{figure}
	\centering
	\includegraphics[width=\columnwidth]{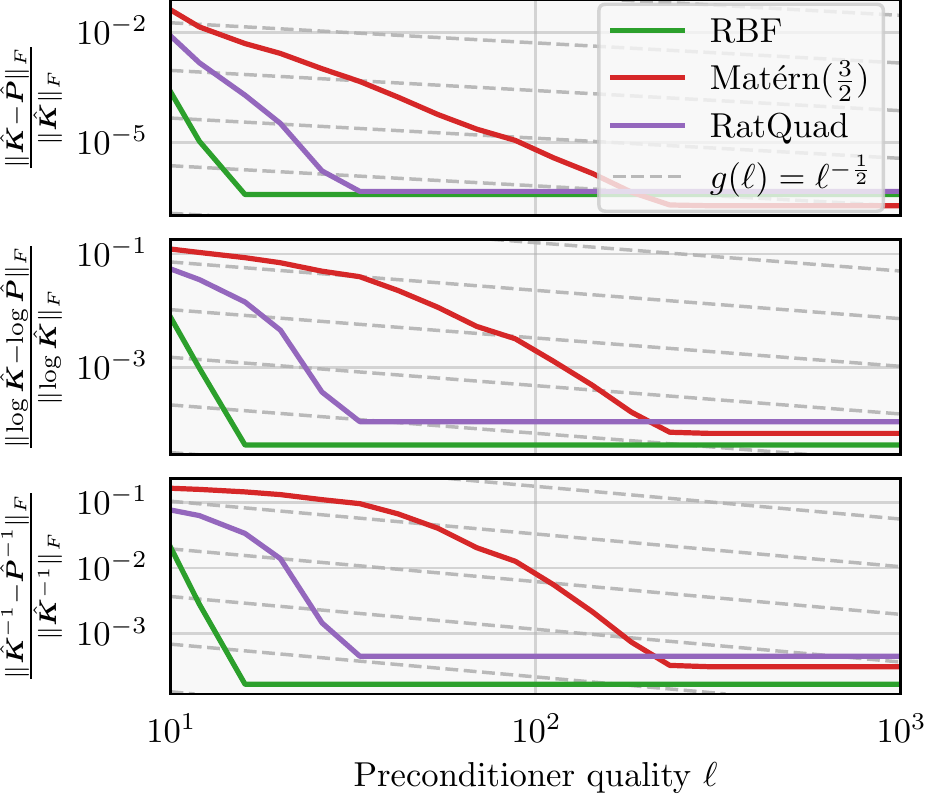}
	\caption{\emph{Relative error of matrix functions.} For analytic functions
		\(f(\shat{\mP_\idxrvs}) \to f(\shat{\mK})\)
		at the asymptotic rate of the preconditioner \(\shat{\mP}_\idxrvs \to \shat{\mK}\).
		Here, we use a partial Cholesky preconditioner on a synthetic dataset (\(n=1,\!000\)).}
	\label{fig:residual-decay}
\end{figure}

\begin{restatable}[Variance-reduced Stochastic Trace Estimation]{theorem}{thmvarreducedtraceestimate}
	\label{thm:variance-reduced-trace-estimate}
	Let \(\shat{\mK}, \shat{\mP}_\idxrvs \in \Rnn_{\mathrm{spd}}, \mDelta_f \in \Rnn\) and \(f:\Rnn_\mathrm{spd} \to \R\) such that \(\tr(f(\shat{\mK})) =
	\operatorname{tr}(f(\shat{\mP}_\idxrvs)) + \tr(\mDelta_f)\), and define the estimator \(\trvarredux =
	\operatorname{tr}(f(\shat{\mP}_\idxrvs)) +	\trSTE(\mDelta_f)\). Now, assume there exist
	\(c_{\mDelta} > 0\) and \(g:\N \to (0, \infty)\) such that
	\begin{equation}
		\norm{\mDelta_f}_F \leq c_{\mDelta}
		g(\idxrvs)\norm{f(\shat{\mK})}_F.\label{eqn:delta_frobenius_bound}
	\end{equation}
	Then there exists \(c_{\vz} > 0\) dependent on the choice of random vectors,
	such that, if \(\idxrvs \geq c_{\vz}\log(\delta^{-1})\), it holds with probability \(1 - \delta \in [\frac{1}{2}, 1)\) that
	\begin{gather}
		\abs{\trvarredux - \tr(f(\shat{\mK}))} \leq \varepsilon_{\idxtxtSTE} \norm{f(\shat{\mK})}_F.\\
		\intertext{where for \(C_1 = c_{\mDelta}\sqrt{c_{\vz}}\) the relative error is given by}
		\varepsilon_{\idxtxtSTE}(\delta, \idxrvs) = C_1 \sqrt{\log(\delta^{-1})}\idxrvs^{-\frac{1}{2}} g(\idxrvs).\label{eqn:varreduced-ste-error}
	\end{gather}
\end{restatable}
\begin{proof}
	See \Cref{suppsec:stochastic-trace-estimation}.
\end{proof}

Notice that \Cref{thm:variance-reduced-trace-estimate} assumes that the sequence \(\{f(\shat{\mP}_\idxrvs)\}_\idxrvs\) approximates \(f(\shat{\mK})\) sufficiently fast with \(\idxrvs\) in \eqref{eqn:delta_frobenius_bound}. Intuitively, if \(\shat{\mP}_\idxrvs \to \shat{\mK}\) quickly, one might expect the same for \(f(\shat{\mP}_\idxrvs) \to f(\shat{\mK})\) under certain conditions on \(f\).
Indeed, one obtains the same asymptotic rate \(g(\idxrvs)\) of the preconditioner \(\shat{\mP}_\idxrvs\) for the approximation of \(f(\shat{\mK})\) by \(f(\shat{\mP})\) (see \Cref{prop:preconditioner_bound_delta_frobenius_bound}). This is illustrated in \Cref{fig:residual-decay}. Therefore, the error of the variance-reduced stochastic trace estimate is \emph{determined by the quality \(g(\idxrvs)\) of the preconditioner}.

\paragraph{Comparison of \Cref{thm:variance-reduced-trace-estimate} and Existing Results}
Consider the case where \(f=\operatorname{id}\). If we are not using a preconditioner, i.e. \(\shat{\mP} = \mZero\) and thus \(c_\mDelta = g(\idxrvs) = 1\), we recover the well-known convergence rate
\(\bigO(\idxrvs^{-\frac{1}{2}})\) of Hutchinson's estimator \cite{Avron2011,RoostaKhorasani2015}.
If instead we choose a randomized low rank
approximation as a preconditioner with \(g(\idxrvs) =
\idxrvs^{-\frac{1}{2}}\), then \Cref{thm:variance-reduced-trace-estimate} recovers the convergence rate \(\varepsilon_{\idxtxtSTE} \in
\bigO(\idxrvs^{-1})\) of \textsc{Hutch++} \cite{Meyer2021,Persson2021,Jiang2021}
as a special case. However, as we will show, using preconditioning one can achieve polynomial -- even exponential -- convergence rates for common kernels. Such a drastic improvement is possible since neither variants of Hutchinson's make any assumptions about the kernel matrix, whereas preconditioners are designed to leverage structure.

\subsection{Forward Pass}
We can now analyze the error of the preconditioned stochastic \(\log\)-determinant estimate.
Combining \Cref{thm:variance-reduced-trace-estimate} with Lanczos quadrature error analysis, the following holds.

\begin{restatable}[Error Bound for \(\log \det(\shat{\mK})\)]{theorem}{convergencelogdet}
	\label{thm:convergence-logdet}
	Let \(f = \log\), \(\mDelta_{\log} = \log\big(\shat{\mP}^{-\frac{1}{2}}
	\shat{\mK} \shat{\mP}^{-\frac{\top}{2}}\big)\) and assume the
	conditions of \Cref{thm:variance-reduced-trace-estimate} hold. Then, with probability \(1-\delta\), it holds for
	\(\trvarreduxlogdet =
	\log(\det(\shat{\mP})) + \trSLQ(\mDelta_{\log})\), that
	\begin{equation*}
		\textstyle \abs{\trvarreduxlogdet - \log\det(\shat{\mK}) } \leq (\varepsilon_{\idxtxtLanczosSTE} +
		\varepsilon_{\idxtxtSTE})\norm{\log(\shat{\mK})}_F,
	\end{equation*}
	where the individual errors are bounded by
	\begin{align}
		\varepsilon_{\idxtxtLanczosSTE}(\kappa, \idxLanczos) & \leq \textstyle K_1\left(\frac{\sqrt{2\kappa + 1} - 1}{\sqrt{2\kappa +1}+1} \right)^{2 \idxLanczos} \\
		\varepsilon_{\idxtxtSTE}(\delta, \idxrvs)         & \leq \textstyle C_1\sqrt{\log(\delta^{-1})}\idxrvs^{-\frac{1}{2}} g(\idxrvs)
	\end{align}
	and \(K_1 =\frac{5 \kappa \log(2(\kappa + 1))}{2\norm{\log(\shat{\mK})}_F \sqrt{2\kappa +1}}\).
\end{restatable}
\begin{proof}
	See \Cref{suppsec:log-determinant}.
\end{proof}

\begin{restatable}{corollary}{proberrorboundlogdet}
	\label{cor:error-bound-logdet}
	Assume the conditions of \Cref{thm:convergence-logdet} hold. If the number of random vectors \(\idxrvs\)
	satisfies \eqref{eqn:varreduced-ste-error} with \(\varepsilon_{\idxtxtSTE} = \frac{\varepsilon}{2}\) and we
	run
	\begin{equation}
		\label{eqn:num_lanczos_steps_logdet}
		\textstyle \idxLanczos \geq
		\frac{\sqrt{3}}{4}\sqrt{\kappa}\log\big(2K_1\varepsilon^{-1}\big)
	\end{equation}
	iterations of Lanczos, then it holds that
	\begin{equation*}
		\boxed{
			\textstyle
			\Pr \left(\abs{\trvarreduxlogdet - \log\det(\shat{\mK})} \leq	\varepsilon \norm{\log(\shat{\mK})}_F\right) \geq
			1-\delta.
		}
	\end{equation*}
\end{restatable}

\begin{proof}
	See \Cref{suppsec:log-determinant}.
\end{proof}

We note two major improvements over the bound by \citet[Corollary~4.5]{Ubaru2017}. First, the number of
Lanczos steps now depends on the condition number \(\kappa\) of
the \emph{preconditioned} matrix, implying faster convergence.
Second, depending on the preconditioner quality \(g(\idxrvs)\), we need significantly
fewer random vectors
by \Cref{thm:variance-reduced-trace-estimate}.

\subsection{Backward Pass}
By differentiating through \eqref{eqn:logdet-decomposition}, we obtain a decomposition into a \emph{deterministic approximation} based
on the
preconditioner and a \emph{residual trace} for the backward pass. For \(\mDelta_{\mathrm{inv}\partial} = \shat{\mK}^{-1} \pdv{\shat{\mK}}{\evtheta} - \shat{\mP}^{-1}
\pdv{\shat{\mP}}{\evtheta}\), we have
\begin{equation}
	\textstyle \pdv{}{\evtheta} \log \det(\shat{\mK}) = \tr(\shat{\mP}^{-1} \pdv{\shat{\mP}}{\evtheta}) +
	\tr(\mDelta_{\mathrm{inv}\partial}),
	\label{eqn:logdet-decomposition_deriv}
\end{equation}
Therefore the stochastic trace estimator
\begin{equation}
	\trSCG(\mDelta_{\mathrm{inv}\partial}) \approx \tr(\mDelta_{\mathrm{inv}\partial})
\end{equation} 
requires solves $\vz_i^\top \mK^{-1} \pdv{\shat{\mK}}{\evtheta} \vz_i$
and $\vz_i^\top \shat{\mP}^{-1} \pdv{\shat{\mP}}{\evtheta} \vz_i$.
The former can be computed with \(\idxCG\) iterations of preconditioned CG,
while the latter is simply a solve with the preconditioner.
Note that the deterministic term $\mathrm{tr}\big(\shat{\mP}^{-1} \pdv{\shat{\mP}}{\evtheta}\big)$
is efficient to calculate for many types of preconditioners. For example, if $\shat{\mP}$
is a diagonal-plus-low-rank preconditioner it can be
computed in \(\bigO(n\idxrvs^2)\) (see
\Cref{suppsec:derivative-log-determinant}).
Using \Cref{thm:variance-reduced-trace-estimate}, we obtain a probabilistic error bound for the derivative estimate.

\begin{restatable}[Error Bound for \(\mathrm{tr}\big(\shat{\mK}^{-1} \pdv{\shat{\mK}}{\evtheta}\big)\)]{theorem}{convergencelogdetbackward}
	\label{thm:convergence-logdet-backward}
	Let \(f(\shat{\mK}) = \shat{\mK}^{-1} \pdv{\shat{\mK}}{\evtheta}\),
	\(\mDelta_{\mathrm{inv}\partial} = \shat{\mK}^{-1} \pdv{\shat{\mK}}{\evtheta} - \shat{\mP}^{-1}
	\pdv{\shat{\mP}}{\evtheta}\) and assume the conditions of \Cref{thm:variance-reduced-trace-estimate} hold. If we solve
	\(\shat{\mK}^{-1} \pdv{\shat{\mK}}{\evtheta}\vz_i\)
	with \(\idxCG\) iterations of preconditioned CG, initialized at \(\vzero\) or better, then it holds
	with probability \(1-\delta\)
	for \(\trvarreduxinvderiv =
	\mathrm{tr}\big(\shat{\mP}^{-1} \pdv{\shat{\mP}}{\evtheta}\big) + \trSCG(\mDelta_{\mathrm{inv}\partial})\), that
	\begin{equation*}
		\textstyle \big \lvert \trvarreduxinvderiv - \mathrm{tr}\big(\shat{\mK}^{-1} \pdv{\shat{\mK}}{\evtheta}\big) \big\rvert
		\leq
		(\varepsilon_{\idxtxtCGSTE} +
		\varepsilon_{\idxtxtSTE}) \norm{\mK^{-1}\pdv{\mK}{\evtheta}}_F,
	\end{equation*}
	where the individual errors are bounded by
	\begin{align}
		\varepsilon_{\idxtxtCGSTE}(\kappa, \idxCG)   & \textstyle \leq K_2\left(\frac{\sqrt{\kappa} - 1}{\sqrt{\kappa}+1} \right)^{\idxCG} \\
		\varepsilon_{\idxtxtSTE}(\delta, \idxrvs) & \textstyle \leq C_1\sqrt{\log(\delta^{-1})}\idxrvs^{-\frac{1}{2}} g(\idxrvs)
	\end{align}
	and \(K_2 = 2 \sqrt{\kappa(\shat{\mK})}n\).
\end{restatable}

\begin{proof}
	See \Cref{suppsec:derivative-log-determinant}.
\end{proof}

\begin{restatable}{corollary}{proberrorboundlogdetbackward}
	\label{cor:error-bound-logdet-backward}
	Assume the conditions of \Cref{thm:convergence-logdet-backward} hold. If the number of random vectors \(\idxrvs\)
	satisfies \eqref{eqn:varreduced-ste-error} with \(\varepsilon_{\idxtxtSTE} = \frac{\varepsilon}{2}\), and we run
	\begin{equation}
		\label{eqn:num-cg-steps-logdet-backward}
		\textstyle \idxCG \geq \frac{1}{2}\sqrt{\kappa}\log(2K_2 \varepsilon^{-1})
	\end{equation}
	iterations of CG, then
	\begin{equation*}
		\boxed{
			\textstyle
			\Pr \left(\big\lvert\trvarreduxinvderiv - \mathrm{tr}\big(\shat{\mK}^{-1} \pdv{\shat{\mK}}{\evtheta}\big)\big\rvert \leq	\varepsilon \norm{\shat{\mK}^{-1}\pdv{\shat{\mK}}{\evtheta}}_F\right) \geq
			1-\delta.
		}
	\end{equation*}

\end{restatable}

\begin{proof}
	See \Cref{suppsec:derivative-log-determinant}.
\end{proof}

\section{Efficient GP Hyperparameter Optimization}
\label{sec:precond-gp-hyperparam-opt}

Having established an efficient way to compute the forward and backward pass for the
\(\log\)-determinant,
we can use these results to accelerate GP hyperparameter optimization by fully exploiting preconditioning not just for the linear solves, but \emph{also} for the \(\log\)-determinant and its derivative.

\subsection{Log-Marginal Likelihood}
We obtain a bound on the \(\log\)-marginal likelihood by combining
\Cref{thm:convergence-logdet} with standard CG convergence analysis.

\begin{restatable}[Error Bound for the \(\log\)-Marginal Likelihood]{theorem}{convergencelogmarginallikelihood}
	\label{thm:convergence-log-marginal-likelihood}
	Assume the conditions of \Cref{thm:convergence-logdet} hold and
	we solve \(\shat{\mK} \vu = \vy\) via preconditioned CG initialized at \(\vu_0\) and
	terminated after
	\(\idxCG\) iterations.
	Then with probability \(1-\delta\), the error in the estimate \(\eta = -\frac{1}{2}(\vy^\top \vu_{\idxCG} + \trvarreduxlogdet + n \log(2\pi))\)
	of the \(\log\)-marginal likelihood \(\logmarglik\) satisfies
	\begin{equation*}
		\textstyle \abs{\eta - \logmarglik } \leq \varepsilon_{\idxtxtCG} +
		\frac{1}{2}(\varepsilon_{\idxtxtLanczosSTE} +
		\varepsilon_{\idxtxtSTE}) \norm{\log(\shat{\mK})}_F,
	\end{equation*}
	where the individual errors are bounded by
	\begin{align}
		\varepsilon_{\idxtxtCG}(\kappa, \idxCG)           & \leq \textstyle K_3\left(\frac{\sqrt{\kappa} - 1}{\sqrt{\kappa}+1} \right)^{\idxCG}                 \\
		\varepsilon_{\idxtxtLanczosSTE}(\kappa, \idxLanczos) & \leq \textstyle K_1\left(\frac{\sqrt{2\kappa + 1} - 1}{\sqrt{2\kappa +1}+1} \right)^{2 \idxLanczos} \\
		\varepsilon_{\idxtxtSTE}(\delta, \idxrvs)         & \leq \textstyle C_1\sqrt{\log(\delta^{-1})}\idxrvs^{-\frac{1}{2}} g(\idxrvs)
	\end{align}
	for \(K_3 = \sqrt{\kappa(\shat{\mK})}
	\norm{\vy}_2 \norm{\vu_0 - \vu}_2\).
\end{restatable}

\begin{proof}
	See \Cref{suppsec:log-marginal-likelihood}.
\end{proof}

\subsection{Derivative of the Log-Marginal Likelihood}
Similarly, we can leverage \Cref{thm:convergence-logdet-backward} for the
derivative.

\begin{restatable}[Error Bound for the Derivative]{theorem}{convergencederivative}
	\label{thm:convergence-derivative}
	Assume the conditions of \Cref{thm:convergence-logdet-backward} hold and
	we solve \(\shat{\mK} \vu = \vy\) via preconditioned CG initialized at \(\vzero\) or
	better and terminated
	after \(\idxCG\) iterations.
	Then with probability \(1-\delta\), the error in the estimate \(\phi = \frac{1}{2}(\vu_{\idxCG}^\top \pdv{\shat{\mK}}{\evtheta} \vu_{\idxCG} - \trvarreduxinvderiv)\)
	of the derivative of the \(\log\)-marginal likelihood \(\pdv{}{\evtheta}\logmarglik\)
	satisfies
	\begin{equation*}
		\textstyle \abs{\phi - \pdv{}{\evtheta}\logmarglik } \leq \varepsilon_{\idxtxtCG} +
		\frac{1}{2}(\varepsilon_{\idxtxtCGSTE} +
		\varepsilon_{\idxtxtSTE})  \norm{\shat{\mK}^{-1}\pdv{\shat{\mK}}{\evtheta}}_F,
	\end{equation*}
	where the individual errors are bounded by
	\begin{align}
		\varepsilon_{\idxtxtCG}(\kappa, \idxCG)        & \leq \textstyle K_4\left(\frac{\sqrt{\kappa} - 1}{\sqrt{\kappa}+1} \right)^{\idxCG}      \\
		\varepsilon_{\idxtxtCGSTE}(\kappa, \idxLanczos) & \leq \textstyle K_2\left(\frac{\sqrt{\kappa} - 1}{\sqrt{\kappa}+1} \right)^{\idxLanczos} \\
		\varepsilon_{\idxtxtSTE}(\delta, \idxrvs)      & \leq \textstyle C_1\sqrt{\log(\delta^{-1})}\idxrvs^{-\frac{1}{2}} g(\idxrvs)
	\end{align}
	for \(K_4 = 6\kappa(\shat{\mK}) \max(\norm{\vu}_2, \norm{\vu}_2^3) \norm{\pdv{\shat{\mK}}{\evtheta}}_2\).
\end{restatable}

\begin{proof}
	See \Cref{suppsec:derivative-log-marginal-likelihood}.
\end{proof}

\begin{table*}
	\caption{\emph{Error rates for combinations of kernels and preconditioners.}
		The rate \(g(\idxrvs)\) measures how fast a sequence of preconditioners
		\(\{\shat{\mP}_{\idxrvs}\}_\idxrvs\) approaches the kernel matrix \(\shat{\mK}\)
		constructed from data \(\mX \in \R^{n \times d}\). Thus it determines both the convergence
		speed of Krylov methods and the preconditioned stochastic trace estimator. Or, equivalently, the faster \(g(\idxrvs) \to 0\) the fewer CG iterations \(\idxCG\) and random vectors \(\idxrvs\)
		are needed to approximate the \(\log\)-marginal likelihood and its gradient.\label{tab:residual-decay-kernels-preconditioners}}
	\centering
	\small
	\begin{tabular}{lclcccc}
    \toprule
    Kernel                  & \(d\)  & Preconditioner      & \(g(\idxrvs)\)                                                            & Condition                                                  & Proof                                         \\
    \midrule
    any & \(\N\) & none & \(1\) & & \Cref{thm:variance-reduced-trace-estimate}\\
    \midrule
    any                     & \(\N\) & truncated SVD       & \(\idxrvs^{-\frac{1}{2}}\)                                                &                                                            & \Cref{prop:truncated-svd-approx-quality}      \\
    any                     & \(\N\) & randomized SVD      & \(\idxrvs^{-\frac{1}{2}} + \bigO(\idxrvs^{\frac{1}{4}}s^{-\frac{1}{4}})\) & w/ high probability for \(s\) samples                            & \Cref{prop:randomized-svd-approx-quality}     \\
    any                     & \(\N\) & randomized Nyström  & \(\idxrvs^{-\frac{1}{2}} + \bigO(\idxrvs^{\frac{1}{4}}s^{-\frac{1}{4}})\) & w/ high probability for \(s\) samples                            & \Cref{prop:nystroem-approx-quality}           \\
    any                     & \(\N\) & RFF                 & \(\idxrvs^{-\frac{1}{2}}\)                                                & w/ high probability                                              & \Cref{prop:rff-approx-quality}                \\
    \midrule
    RBF                     & \(1\)  & partial Cholesky & \(\exp(-c\idxrvs)\)                                                       & for some \(c>0\)                                           & \Cref{prop:cholesky-approx-quality}           \\
    RBF                     & \(\N\) & QFF                 & \({\exp}(-b \idxrvs^\frac{1}{d})\)                                        & for some \(b>0\) if \(\idxrvs^\frac{1}{d} > 2\gamma^{-2}\) & \Cref{prop:qff-approx-quality}                \\
    Mat\' ern(\(\nu\))      & \(\N\) & partial Cholesky & \(\idxrvs^{-(\frac{2\nu}{d}+1)}\)                                         & \(2\nu \in \N\) and  maximin ordering                      & \cite{Schaefer2021a}                           \\
    Mat\' ern(\(\nu\))      & \(1\)  & QFF                 & \(\idxrvs^{-(s(\nu) + 1)}\)                                               & where \(s(\nu) \in \N\)                                    & \Cref{prop:general-qff-approx-quality}        \\
    mod. Mat\' ern(\(\nu\)) & \(\N\) & QFF                 & \(\idxrvs^{-\frac{s(\nu)+1}{d}}\)                                         & where \(s(\nu) \in \N\)                                    & \Cref{prop:general-qff-approx-quality}        \\
    \midrule
    additive                & \(\N\) & any                 & \(d g(\idxrvs)\)                                                          & all summands have rate \(g(\idxrvs)\)                          & \Cref{lem:additive-kernel-approx-quality}     \\
    any                     & \(\N\) & any kernel approx.   & \(g(\idxrvs)\)                                                            & \(\exists\) uniform convergence bound                      & \Cref{lem:uniform-convergence-approx-quality} \\
    \bottomrule
\end{tabular}
	\vspace{-1em}
\end{table*}

\subsection{Preconditioner Choice}
Our theoretical convergence results fundamentally depend on how quickly the preconditioner approximates the kernel matrix, either directly via \(g(\idxrvs)\), or indirectly via the condition number improvement \eqref{eqn:condition-preconditioner-quality}. This leaves the question which preconditioners should be chosen in practice and what rates \(g(\idxrvs)\) they attain. In \Cref{tab:residual-decay-kernels-preconditioners}, we
give an extensive list of
kernel-preconditioner combinations with associated rates (see \Cref{suppsec:preconditioning} for proofs). This includes the commonly used RBF and Mat{\'e}rn(\(\nu\)) kernels for which the Cholesky \cite{Kershaw1978} and
QFF \cite{Mutny2018} preconditioners result in exponential and polynomial convergence rates, respectively. For STE in this context this is a substantial improvement over
the rate of Hutchinson's estimator \(\bigO(\idxrvs^{-\frac{1}{2}})\) \cite{Avron2011,RoostaKhorasani2015,Skorski2021} and \textsc{Hutch++} with \(\bigO(\idxrvs^{-1})\)
\cite{Meyer2021,Persson2021,Jiang2021}. Depending on the problem this can mean a difference of tens vs. thousands of random vectors. To the best of our knowledge, for the use of CG in GP inference, only the one-dimensional RBF kernel and partial Cholesky preconditioner have been previously analyzed theoretically \cite{Gardner2018}. In contrast, \Cref{tab:residual-decay-kernels-preconditioners} gives convergence rates for arbitrary \(d\)-dimensional kernels and multiple preconditioners. In fact, our results also apply to any
kernel approximation with a uniform convergence bound (such as RFF \cite{Rahimi2007}). All the while for many, e.g. diagonal-plus-low-rank preconditioners, the amount of precomputation needed \emph{amortizes with more data}, i.e. the cost of preconditioning becomes negligible the larger the dataset.

\subsection{Algorithms}
The above leads to \Cref{alg:log_marginal_likelihood,alg:derivative_log_marginal_likelihood} computing \(\logmarglik\) and
\(\pdv{}{\evtheta}\logmarglik\) for GP hyperparameter
optimization.\footnote{While presented sequentially for clarity, in practice one would presample all random vectors and
	run a single call of (parallelized) CG with multiple right hand sides, as in \citep{Gardner2018}.
}
Our algorithms are similar to those presented in prior work by
\citet{Cutajar2016,Ubaru2017,Gardner2018},
yet crucially they leverage preconditioning
for \textcolor{precondcgcolor}{faster CG convergence} \emph{and} \textcolor{varreductioncolor}{variance reduction} of the $\log$-determinant estimate and its derivative. In the following, \(\textsc{CG}(\shat{\mK}, \vy, \shat{\mP}, \idxCG)\) denotes a CG solve of
\(\shat{\mK}\vu =\vy\) with preconditioner
\(\shat{\mP}\) run for \(\idxCG\) iterations.
Here, we equivalently use CG instead of Lanczos, as suggested by
\citet{Gardner2018}.
\vspace{-1em}

\begin{algorithm}
	\caption{\(\log\)-Marginal Likelihood\label{alg:log_marginal_likelihood}}
	\small
	\begin{algorithmic}[1]
		\Require{ $\vy$ (labels), $\shat \mK$ (kernel matrix), $\shat \mP$ (preconditioner),
			$\idxrvs$ ($\#$ of random STE vectors), $\idxCG$ ($\#$ of CG iterations) }
		\Procedure{\textsc{LogMargLikelihood}}{$\vy, \shat{\mK}, \shat{\mP}, \idxrvs, \idxCG$}
		\State \(\vu \gets \textsc{CG}(\shat{\mK}, \vy, \textcolor{precondcgcolor}{\shat{\mP}}, \idxCG)\) \Comment{\(\approx \shat{\mK}^{-1} \vy\)}
		\textcolor{varreductioncolor}{\State \(\trprecondlogdet \gets \log \det(\shat{\mP})\)}
		\For{\(i = 1, \dots, \idxrvs\)} 
		\State \(\vz_i \gets \tilde{\vz}_i / \norm{\tilde{\vz}_i}_2\) for rand. vector
		\(\tilde{\vz}_i\)
		\State \(\textcolor{varreductioncolor}{\mT} \gets \textsc{CG}(\shat{\mK}, \vz_i, \textcolor{precondcgcolor}{\shat{\mP}}, \idxCG)\) \Comment{equiv. to \textsc{Lanczos}}
		\State \([\textcolor{varreductioncolor}{\mW}, \textcolor{varreductioncolor}{\vlambda}] \gets \textsc{EigenDecomp}(\mT)\)
		\Comment{\(\mT\) tridiagonal}
		\State \(\textcolor{varreductioncolor}{\omega_j} \gets (\ve_1^\top \textcolor{varreductioncolor}{\vw_j})^2\) for \(j = 0, \dots, \idxCG\)
		\Comment{quad. weights}
		\State \(\gamma_i \gets \sum_{j=0}^\idxCG \textcolor{varreductioncolor}{\omega_j} \log(\textcolor{varreductioncolor}{\evlambda_j})\)
		\Comment{\(\approx \vz_i^\top \mDelta_{\log} \vz_i\)}
		\EndFor
		\State \(\trvarreduxlogdet \gets \textcolor{varreductioncolor}{\trprecondlogdet} + \frac{n}{\idxrvs} \sum_{i=1}^\idxrvs \gamma_i\)
		\Comment{\(\approx \log\det(\shat{\mK})\)}
		\State \Return {$-\frac{1}{2}(\vy^\top \vu + \trvarreduxlogdet + n \log(2\pi))$} \Comment{\(\approx \logmarglik(\vtheta)\)}
		\EndProcedure
	\end{algorithmic}
\end{algorithm}
\begin{algorithm}
	\caption{Derivative of the \(\log\)-Marginal Likelihood\label{alg:derivative_log_marginal_likelihood}}
	\small
	\begin{algorithmic}[1]
		\Require{ $\vy$ (labels), $\shat \mK$ (kernel matrix), $\shat \mP$ (preconditioner),
			$\idxrvs$ ($\#$ of random STE vectors), $\idxCG$ ($\#$ of CG iterations),
			$\pdv{\shat{\mK}}{\evtheta}()$ / $\pdv{\shat{\mP}}{\evtheta}()$ (functions for computing kernel / preconditioner
			derivatives)}
		\Procedure{\textsc{Derivative}}{$\vy, \shat{\mK}, \pdv{\shat{\mK}}{\evtheta}(), \shat{\mP},
				\pdv{\shat{\mP}}{\evtheta}(), \idxrvs, \idxCG$}
		\State \(\vu \gets \textsc{CG}(\shat{\mK}, \vy, \textcolor{precondcgcolor}{\shat{\mP}}, \idxCG)\) \Comment{\(\approx \shat{\mK}^{-1} \vy\)}
		\State \textcolor{varreductioncolor}{\(\trprecondinvderiv \gets \tr(\shat{\mP}^{-1} \pdv{\shat{\mP}}{\evtheta})\)}
		\For{\(i = 1, \dots, \idxrvs\)} 
		\State \(\vz_i \gets \tilde{\vz}_i / \norm{\tilde{\vz}_i}_2\) for rand. vector
		\(\tilde{\vz}_i\)
		\State \(\vw_i \gets \textsc{CG}(\shat{\mK}, \pdv{\shat{\mK}}{\evtheta}\vz_i, \textcolor{precondcgcolor}{\shat{\mP}}, \idxCG)\) \Comment{\(\approx \shat{\mK}^{-1}\pdv{\shat{\mK}}{\evtheta}\vz_i\)}
		\State \textcolor{varreductioncolor}{\(\tilde{\vw}_i \gets \shat{\mP}^{-1}\pdv{\shat{\mP}}{\evtheta}\vz_i\)}
		\State \(\gamma_i \gets \vz_i^\top(\vw_i \textcolor{varreductioncolor}{- \tilde{\vw}_i})\) \Comment{\(\approx \vz_i^\top \mDelta_{\mathrm{inv}\partial} \vz_i\)}
		\EndFor
		\State \(\trvarreduxinvderiv \gets \textcolor{varreductioncolor}{\trprecondinvderiv} + \frac{n}{\idxrvs} \sum_{i=1}^\idxrvs \gamma_i\)
		\Comment{\(\approx \tr(\mK^{-1} \pdv{\mK}{\evtheta})\).}
		\State \Return {$\frac{1}{2}(\vu^\top \pdv{\mK}{\evtheta} \vu - \trvarreduxinvderiv)$} \Comment{\(\approx \pdv{}{\evtheta}\logmarglik(\vtheta)\)}
		\EndProcedure
	\end{algorithmic}
\end{algorithm}

\begin{figure*}[th!]
	\centering
	\includegraphics[width=\textwidth]{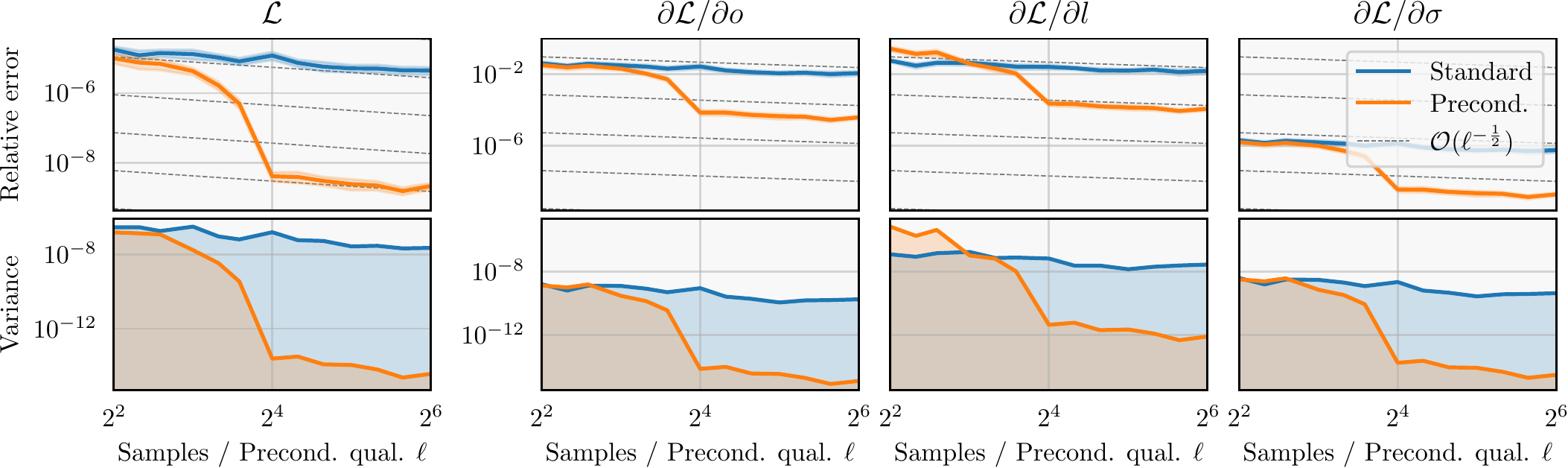}
	\caption{\emph{Bias and variance of the estimators for the \(\log\)-marginal likelihood \(\logmarglik\) and its
			derivatives.} The relative error and variance decrease faster with the number of random
		vectors \(\idxrvs\) when using a preconditioner \(\shat{\mP}_\idxrvs\). The decrease
		rate \(\bigO(\idxrvs^{-\frac{1}{2}}g(\idxrvs))\), determined by the preconditioner, significantly improves upon the standard Hutchinson's rate \(\bigO(\idxrvs^{-\frac{1}{2}})\).}
	\label{fig:synthetic-data}
\end{figure*}

\paragraph{Computational Complexity}
\Cref{alg:log_marginal_likelihood} has complexity \(\bigO(n^2\idxCG\idxrvs + P_{\log\det})\)
and \Cref{alg:derivative_log_marginal_likelihood} has complexity \(\bigO((n^2 \idxCG +
P_{\mathrm{solve}})\idxrvs +
P_{\tr \mathrm{inv}\partial})\), where
\(P_{(\cdot)}\)
denotes the cost of an operation with the preconditioner.\footnote{For diagonal-plus-low-rank preconditioners,
	such as the partial Cholesky, \(P_{\mathrm{solve}}\),
	\(P_{\log\det}\), and \(P_{\tr \mathrm{inv}\partial}\) are in \(\bigO(n\idxrvs^2)\) by the matrix
	inversion and
	determinant lemmas.}
Assuming \(\idxCG, \idxrvs \ll n\),
this is asymptotically faster than Cholesky-based inference with complexity
\(\bigO(n^3)\).
Due to the reduction to matrix-vector multiplication, if \(\vv \mapsto \shat{\mK}
\vv\) is more efficient than \(\bigO(n^2)\) (e.g. for structured or sparse matrices)
the complexity reduces further. Finally, the \textbf{for}-loops are embarrassingly
parallel, giving additional speedup in practice.

\subsection{Related Work}
Krylov methods have been used for GP inference since the work of \citet{Gibbs1997}.
While these methods were primarily relegated to structured GPs that afford fast matrix-vector
products \citep{Cunningham2008,Saatci2012,Wilson2015a},
they have seen growing use as a general purpose method, especially when coupled with specialized, parallel
hardware \citep{Murray2009,Anitescu2012,Gardner2018,Wang2019,Artemev2021}. Preconditioners can be used to accelerate and stabilize the necessary linear solves with the kernel matrix \cite{Faul2005,Gumerov2007,Stein2012,Chen2013,Cutajar2016}.
To compute the $\log$-determinant of the kernel matrix, some recent works propose variance-free (but biased) estimates \citep[e.g.][]{Artemev2021},
though many works compute this term by combining STE \cite{Hutchinson1989,Bekas2007,Avron2011,RoostaKhorasani2015} with SLQ
\cite{Golub2009,Ubaru2017,Dong2017,Gardner2018,Cortinovis2021}.
Our work builds on ideas for variance-reduced
stochastic trace esimation
\cite{Adams2018,Meyer2021,Persson2021,Jiang2021}, but, by leveraging preconditioning, requires significantly fewer random
vectors than existing approaches.
When applied to GP hyperparameter optimization,
we obtain stronger theoretical guarantees
for the forward pass than previously known \cite{Ubaru2017} and novel guarantees for the backward pass. Finally, our results on preconditioners for kernel
matrices (\Cref{tab:residual-decay-kernels-preconditioners}) give a rigorous foundation to their use for GPs as proposed by \citet{Cutajar2016} and others.

\section{Experiments}
\label{sec:experiments}

\begin{figure*}
	\centering
	\begin{subfigure}[t]{0.32\textwidth}
		\centering
		\includegraphics[width=\textwidth]{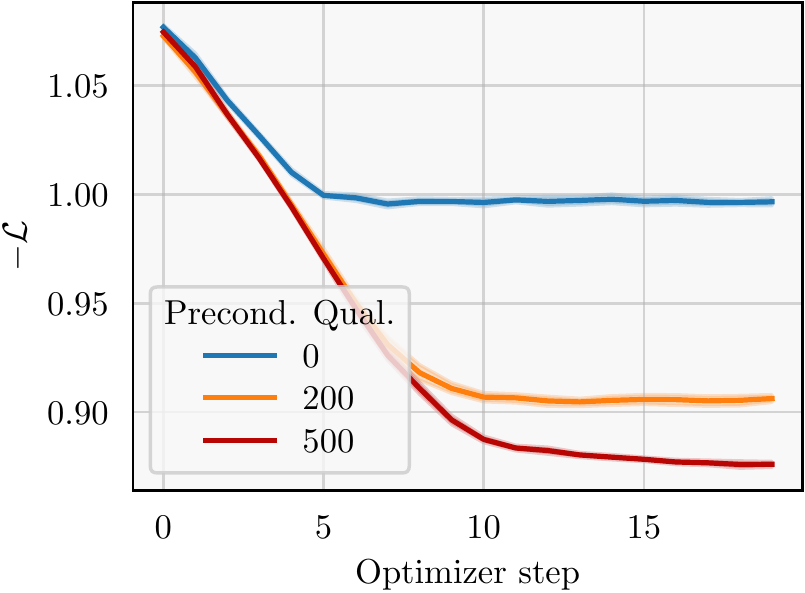}
		\caption{Training loss for ``Protein''.}
		\label{subfig:protein-training-loss}
	\end{subfigure}%
	\hspace{0.5em}
	\begin{subfigure}[t]{0.32\textwidth}
		\centering
		\includegraphics[width=\textwidth]{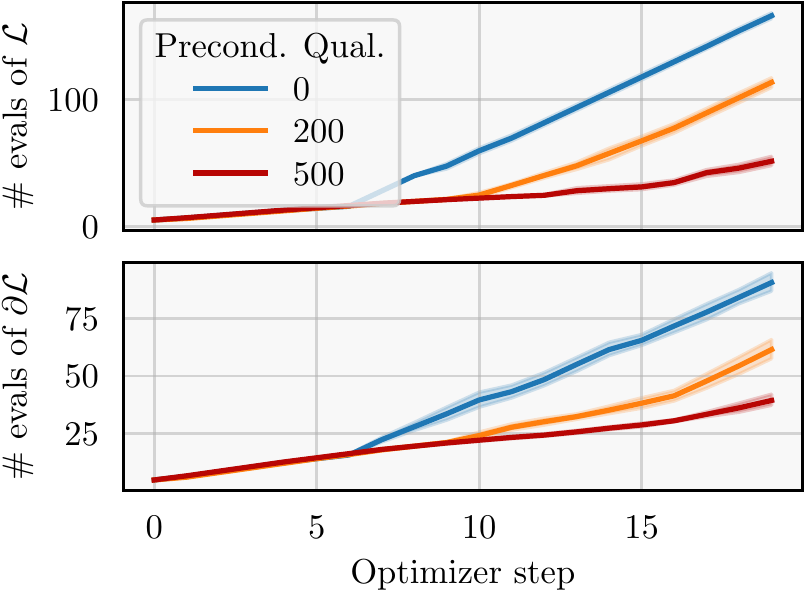}
		\caption{Line search computations for ``Protein''.}
		\label{subfig:protein-linesearch}
	\end{subfigure}%
	\hspace{0.5em}
	\begin{subfigure}[t]{0.32\textwidth}
		\centering
		\includegraphics[width=\textwidth]{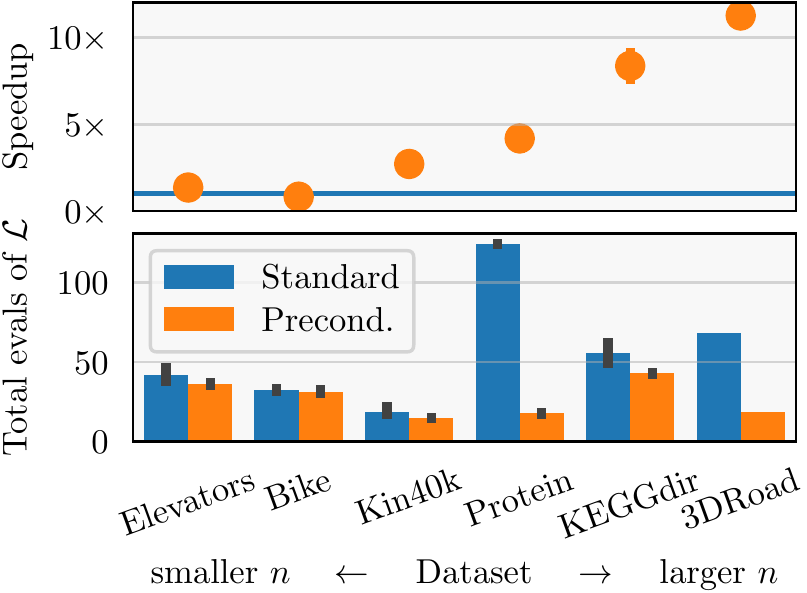}
		\caption{Speedup on UCI datasets.}
		\label{subfig:uci-speedup}
	\end{subfigure}%
	\caption{\emph{Preconditioning reduces noise and in turn accelerates hyperparameter optimization.} Variance reduction improves optimization via better search directions and
		fewer evaluations of \(\logmarglik\) and \(\pdv{}{\evtheta_i}\logmarglik\) for the line
		search. \subref{subfig:protein-training-loss} Training loss and \subref{subfig:protein-linesearch} model evaluations for
		line search decrease with preconditioner size, as shown for the ``Protein'' dataset. \subref{subfig:uci-speedup} The reduction in loss (and gradient)
		evaluations and of noise in the gradients results in an order of magnitude speedup on UCI datasets.}
	\label{fig:uci-results}
\end{figure*}

\begin{table*}[t]
	\caption{\emph{Hyperparameter optimization on UCI datasets.} GP regression using a Mat\'ern\((\frac{3}{2})\) kernel and
		partial Cholesky preconditioner of size \(500\) with \(\idxrvs=50\) random samples.
		Hyperparameters
		were optimized with L-BFGS for at most 20 steps using early stopping. 
		All results, but ``3DRoad'', are averaged over 10 runs. Differences \(\geq 1\) standard deviation in bold.}
	\label{tab:uci-results}
	\small
	\centering
	\renewrobustcmd{\bfseries}{\fontseries{b}\selectfont}
	\sisetup{detect-weight=true,detect-inline-weight=math}
	\begin{tabular}{cS[table-format=6]S[table-format=2]S[table-format=-1.4,round-mode=places,round-precision=4]S[table-format=-1.4,round-mode=places,round-precision=4]S[table-format=-1.4,round-mode=places,round-precision=4]S[table-format=-1.4,round-mode=places,round-precision=4]S[table-format=-1.4,round-mode=places,round-precision=4]S[table-format=-1.4,round-mode=places,round-precision=4]S[table-format=5,round-mode=places,round-precision=0]S[table-format=5,round-mode=places,round-precision=0]}
\toprule
{Dataset} &  {$n$} & {$d$} & \multicolumn{2}{c}{{$-\logmarglik_{\mathrm{train}} \downarrow$}} & \multicolumn{2}{c}{{$-\logmarglik_{\mathrm{test}} \downarrow$}} & \multicolumn{2}{c}{{RMSE $\downarrow$}} & \multicolumn{2}{c}{{Runtime (s)}} \\
       {} &     {} &    {} &                                  {Standard} & {Precond.} &                                 {Standard} & {Precond.} &          {Standard} & {Precond.} &    {Standard} &  {Precond.} \\
\midrule
Elevators &  12449 &    18 &                                    0.464722 &   \bfseries0.437725 &                                   0.402140 &   0.402184 &            0.348366 &   0.348241 &     53.000000 &   \bfseries39.181818 \\
     Bike &  13034 &    17 &                                   -0.997622 &  -0.998517 &                                  -0.993428 &  -0.987725 &            0.044620 &   0.045370 &     \bfseries30.636364 &   37.090909 \\
   Kin40k &  30000 &     8 &                                   -0.333929 &  \bfseries-0.433196 &                                  -0.314085 &  -0.313514 &            \bfseries0.092942 &   0.094906 &    186.545455 &   \bfseries44.636364 \\
  Protein &  34297 &     9 &                                    0.996320 &   \bfseries0.927287 &                                   0.886924 &   0.883540 &            0.572161 &   \bfseries0.557747 &    892.636364 &   \bfseries42.545455 \\
  KEGGdir &  36620 &    20 &                                   -0.950094 &  \bfseries-1.004278 &                                  -0.945906 &  -0.948952 &            0.086087 &   0.086368 &   1450.272727 &  \bfseries173.727273 \\
   3DRoad & 326155 &     3 &                                    0.773300 &   \bfseries0.128400 &                                   1.436000 &   \bfseries1.169000 &            0.298200 &   \bfseries0.126500 &  82200.000000 & \bfseries7306.000000 \\
\bottomrule
\end{tabular}

\end{table*}

We validate our theoretical findings empirically via GP hyperparameter optimization on synthetic and
benchmark datasets with and without preconditioning. We find that
\begin{enumerate}[noitemsep,topsep=0pt,label=(\alph*)]
	\item \emph{preconditioning reduces bias and variance} in the forward and backward pass, which results in
	\item \emph{less noisy search directions} and \emph{fewer \(\log\)-likelihood and gradient evaluations for the line search}.
\end{enumerate}
This allows the use of rapidly converging optimizers, and
\begin{enumerate}[noitemsep,topsep=0pt,label=(\alph*),start=3]
	\item \emph{accelerates training significantly.}.
\end{enumerate}

\paragraph{Experimental Setup}
We consider a one-dimensional synthetic dataset of \(n=10,\!000\) iid standard normal samples, as
well as a range of UCI datasets \cite{Dua2017} with
training set sizes ranging from \(n=12,\!449\) to \(326,\!155\) (see \Cref{tab:uci-results}).
All experiments were performed on single NVIDIA GPUs, a GeForce RTX 2080 and Titan RTX,
respectively. We perform GP regression using an RBF and
Mat{\'e}rn\((\frac{3}{2})\) kernel with output scale \(o\),
lengthscales \(l_j\) -- one per input dimension -- and noise \(\sigma^2\). Hyperparameters
were optimized with L-BFGS using an
Armijo-Wolfe line search and early stopping via a validation set. We use a partial Cholesky
preconditioner throughout. An implementation of our method is available as part of 
\textsc{GPyTorch} \cite{Gardner2018}.\footnote{\href{https://github.com/cornellius-gp/gpytorch}{\texttt{github.com/cornellius-gp/gpytorch}}}

\paragraph{Preconditioning reduces bias \& variance in \(\logmarglik\) and
	\(\pdv{}{\evtheta}\logmarglik\)}
\Cref{fig:synthetic-data} shows the relative error of the marginal \(\log\)-likelihood and its
derivatives on synthetic data. Already for \(\idxrvs \geq 16\) random samples
\emph{bias and variance are reduced by several orders of magnitude}.
We observe exponential decrease and then a return to the standard Hutchinson's rate of
\(\mathcal{O}(\idxrvs^{-\frac{1}{2}})\). After \(\idxrvs=16\) iterations the algorithm computing the
preconditioner
has reached a specified tolerance and terminates, invalidating the approximation quality assumption
\eqref{eqn:preconditioner-quality} for \(\idxrvs \geq 16\).
Similar observations hold for the Mat\'ern and RatQuad kernel (see \Cref{tab:synthetic-data} and
\Cref{fig:synthetic-data-detailed}). As predicted by \Cref{thm:variance-reduced-trace-estimate} and illustrated by
\Cref{fig:synthetic-data}, the variance reduction is
determined by the preconditioner. For higher
dimensions, the rate \(g(\idxrvs)\) generally slows (see \Cref{tab:residual-decay-kernels-preconditioners}), which in turn
reduces the
bias and variance reduction achieved by our method (see \Cref{tab:synthetic-data}).
However, on real datasets we still see strong variance reduction via our method, possibly since
real data often lies on a low-dimensional manifold.

\paragraph{Preconditioning accelerates hyperparameter optimization}
On datasets from the UCI repository, we find that preconditioning results
in lower training loss \(-\logmarglik(\vtheta)\) (illustrated in \Cref{subfig:protein-training-loss}) on
almost all datasets and essentially
identical generalization error (see \Cref{tab:uci-results}). Reducing stochasticity via preconditioning significantly lowers
the number of \(\logmarglik\) and
\(\pdv{}{\evtheta}\logmarglik\) evaluations for the line search during
optimization
(see \Cref{subfig:protein-linesearch}) and results in less noisy search directions. In fact, the noise in the loss and gradients caused by
stochastic trace estimation previously necessitated the use of slower converging, but more noise-robust
optimizers \cite{Wang2019}, such as Adam \cite{Kingma2015}. As these experiments show,
our variance-reduced estimators make the use of L-BFGS possible, which significantly outperforms Adam (c.f. \Cref{tab:uci-results} and \Cref{tab:uci-results-adam}). These combined effects due to preconditioning \emph{accelerate training up to twelvefold}, as \Cref{subfig:uci-speedup} shows.
We observe that the speedup increases with the size of the dataset. This is partly explained by the amortization of the cost of computing and applying the preconditioner with increasing \(n\).

\section{Conclusion}
One might reasonably hope that structural knowledge about the kernel matrix can accelerate GP hyperparameter optimization. Preconditioning is a way to encode and exploit such
structure. As we showed, it can be used to great effect -- not only for the solution of
linear systems -- but importantly \emph{also} for stochastic approximation of the \(\log\)-determinant and its derivative. Our convergence results combined with the rates for kernel-preconditioner pairs in \Cref{tab:residual-decay-kernels-preconditioners} rigorously explain why preconditioning has been observed empirically to be so effective for large-scale GP inference \cite{Cutajar2016,Gardner2018,Wang2019}. 

In fact, our work implies that software packages for GPs, which make use of Krylov methods for inference, should not use a fixed
preconditioner. Instead, the preconditioner should be automatically chosen based on the specified model. While we
derive a range of such kernel and preconditioner combinations, it is
likely that better preconditioners exist for certain kernels or types of data.
Other scientific fields invest substantial research effort into the design of preconditioners, e.g. for PDEs \cite{Saad2003}.
Our work strongly suggests that, similarly, developing specialized
preconditioners is a promising approach to scale Gaussian processes.

\section*{Acknowledgements}
JW and PH gratefully acknowledge financial support by the European Research Council through ERC
StG Action 757275 / PANAMA; the DFG Cluster of Excellence ``Machine Learning - New Perspectives for
Science'', EXC 2064/1, project number 390727645; the German Federal Ministry of Education and
Research (BMBF) through the T\" ubingen AI Center (FKZ: 01IS18039A); and funds from the Ministry of
Science, Research and Arts of the State of Baden-W\" urttemberg. JW is grateful to the
International Max Planck Research School for Intelligent Systems (IMPRS-IS) for support. GP and JPC
are supported by the Simons Foundation, McKnight Foundation, the Grossman Center, and the Gatsby
Charitable Trust.

The authors would like to thank Marius Hobbhahn, Lukas Tatzel and Felix Dangel for helpful feedback
on an earlier version of this manuscript.

	{\small
		\bibliographystyle{icml2022}
		\bibliography{references}
	}

\clearpage




\beginsupplementary
\startcontents[supplementary]

\onecolumn

This supplementary material is structured as follows. \Cref{suppsec:background-krylov-methods} contains background on Krylov methods, such as known convergence results. \Cref{suppsec:stochastic-trace-estimation} contains the main result and proof for variance-reduced stochastic trace estimation. \Cref{suppsec:log-determinant-estimation} gives proofs for the forward and backward pass of the approximation to the \(\log\)-determinant. In turn, \Cref{suppsec:gp-hyperparameter-optimization} contains the error bounds for the \(\log\)-marginal likelihood and its derivative. Error rates for specific preconditioners are given in \Cref{suppsec:preconditioning} and finally, additional experimental results can be found in \Cref{suppsec:additional-experimental-results}.

References referring to sections, equations or theorem-type
environments within the supplement are prefixed with `S', while references to, or results from, the
main paper are stated as is.

	{\small
		\vspace{1em}
		\printcontents[supplementary]{}{1}{}
	}

\section{Background on Krylov Methods}
\label{suppsec:background-krylov-methods}

\subsection{Conjugate Gradient Method}
\label{sec:conjugate-gradient-method}

\begin{theorem}[Convergence Rate of Preconditioned CG \cite{Trefethen1997}]
	\label{thm:convergence_rate_cg}
	Let \(\mA, \mP \in \Rnn\) be symmetric positive definite. The error of the conjugate gradient
	method with preconditioner \(\mP\) after \(\idxCG \in \N\) steps is given by
	\begin{align}
		\norm{\vx_k - \vx}_\mA \leq 2 \left( \frac{\sqrt{\kappa} - 1}{\sqrt{\kappa} + 1} \right)^\idxCG \norm{\vx_0 - \vx}_\mA \\
		\intertext{and in euclidean norm by}
		\norm{\vx_k - \vx}_2 \leq 2 \sqrt{\kappa(\mA)} \left( \frac{\sqrt{\kappa} - 1}{\sqrt{\kappa} + 1} \right)^\idxCG \norm{\vx_0 - \vx}_2
	\end{align}
	where \(\kappa = \kappa(\mP^{-\frac{1}{2}}\mA\mP^{-\frac{^\top}{2}})\) is the condition number of the preconditioned system
	matrix.
\end{theorem}

\begin{proof}
	Preconditioned CG is equivalent to running CG on the transformed problem
	\begin{equation*}
		\tilde{\mA} \tilde{\vx} =
		\mP^{-\frac{1}{2}}\mA\mP^{-\frac{^\top}{2}}
		\tilde{\vx} = \mP^{-\frac{1}{2}}\vb
	\end{equation*}
	with the substitution \(\tilde{\vx} = \mP^\frac{\top}{2} \vx\). By
	\citet{Trefethen1997}, the
	convergence rate of CG on the problem is given by
	\begin{equation*}
		\norm{\tilde{\vx}_\idxCG - \tilde{\vx}}_{\tilde{\mA}} \leq 2 \left( \frac{\sqrt{\kappa} - 1}{\sqrt{\kappa} + 1} \right)^\idxCG \norm{\tilde{\vx}_0 - \tilde{\vx}}_{\tilde{\mA}}
	\end{equation*}
	The first equation follows by recognizing that
	\begin{equation*}
		\norm{\tilde{\vx}_\idxCG - \tilde{\vx}}_{\tilde{\mA}}^2 = (\tilde{\vx}_\idxCG - \tilde{\vx})^\top
		\mP^{-\frac{1}{2}}\mA\mP^{-\frac{\top}{2}} (\tilde{\vx}_\idxCG - \tilde{\vx}) = (\vx_\idxCG - \vx)
		\mA (\vx_\idxCG - \vx) = \norm{\vx_\idxCG - \vx}_\mA^2.
	\end{equation*}
	Now it holds by the min-max principle, that
	\begin{equation*}
		\sqrt{\lambda_{\min}({\tilde{\mA}})} \norm{\vx_k - \vx}_2 \leq
		\norm{{\vx}_\idxCG -
		{\vx}}_{{\mA}} \leq 2 \left( \frac{\sqrt{\kappa} - 1}{\sqrt{\kappa} + 1} \right)^\idxCG
		\norm{\vx_0 - \vx}_{{\mA}} 			      \leq
		2 \sqrt{\lambda_{\max}({{\mA}})} \left( \frac{\sqrt{\kappa} - 1}{\sqrt{\kappa} + 1} \right)^{\idxCG}
		\norm{\vx_0 - \vx}_2.
	\end{equation*}
\end{proof}

\begin{corollary}
	\label{cor:iterations_cg}
	Let \(\varepsilon \in (0, 1]\), then preconditioned CG has relative error
	\(\norm{\vx_k - \vx}_\mA \leq \varepsilon \norm{\vx_0 - \vx}_\mA\) after
	\begin{equation}
		\idxCG \geq \frac{\sqrt{\kappa}}{2} \log(2\varepsilon^{-1})
	\end{equation}
	iterations, where \(\kappa\) is the condition number of the preconditioned system matrix. In
	euclidean norm \(\norm{\cdot}_2\) relative error \(\varepsilon\) is achieved after
	\begin{equation}
		\idxCG
		\geq \frac{\sqrt{\kappa}}{2}\log(2 \sqrt{\kappa(\mA)} \varepsilon^{-1})
	\end{equation}
	iterations.
\end{corollary}
\begin{proof}
	It holds by \Cref{lem:upper-bound-conv-frac} and the assumption on the number of iterations \(m\), that
	\begin{equation*}
		2 \left(\frac{\sqrt{\kappa}-1}{\sqrt{\kappa}+1}\right)^m \leq 2 \exp(- \frac{2}{\sqrt{\kappa}}m) \leq 2 \exp(-\log(\frac{2}{\varepsilon})) =
		\varepsilon
	\end{equation*}
	Using \Cref{thm:convergence_rate_cg} proves the statement. The proof for the euclidean norm is analogous.
\end{proof}

\subsection{Lanczos Algorithm}
The \emph{Lanczos algorithm} \cite{Lanczos1950} is a Krylov method, which for a
symmetric matrix \(\mA \in \Rnn\) iteratively builds an approximate tridiagonalization
\begin{equation*}
	\mA \approx \tilde{\mQ} \tilde{\mT} \tilde{\mQ}
\end{equation*}
where \(\tilde{\mQ} \in \R^{n \times \idxLanczos}\) orthonormal and
\(\tilde{\mT} \in \R^{\idxLanczos \times \idxLanczos}\) tridiagonal. For an initial probe vector
\(\vb \in \Rn\), Gram-Schmidt orthogonalization is applied to the Krylov subspace basis. The
orthogonalized vectors form \(\tilde{\mQ}\), while the Gram-Schmidt coefficients form
\(\tilde{\mT}\). This low-rank approximation becomes an exact tridiagonalization \(\mA
= \mQ \mT \mQ^\top\) for \(\idxLanczos=n\). The Lanczos process is often used to compute
(approximate)
eigenvalues and eigenvectors, which is done by computing an eigendecomposition of the tridiagonal
matrix \(\tilde{\mT}\) at cost \(\bigO(\idxLanczos^2)\). The tridiagonal matrix \(\tilde{\mT}\) can also be formed by running CG on the linear system \(\mA \vx = \vb\) and by collecting the step lengths
\(\alpha_i\) and conjugacy corrections \(\beta_i\) used in the solution and search direction
updates \citep[Section~6.7.3]{Saad2003}.

\subsection{Stochastic Lanczos Quadrature}

One can approximate \(\tr(f(\mA))\) for symmetric positive definite \(\mA\) via \emph{stochastic Lanczos quadrature} (SLQ) \cite{Golub2009,Ubaru2017} by combining Hutchinson's estimator with quadrature and the Lanczos algorithm. It holds that
\begin{align*}
	\tr(f(\mA)) \approx \trSTE(f(\mA)) = \frac{n}{\idxrvs} \sum_{i=1}^\idxrvs \vz_i^\top f(\mA)\vz_i \approx \frac{n}{\idxrvs} \sum_{i=1}^\idxrvs I_\idxLanczos^{(i)} = \trSLQ(f(\mA))
\end{align*}
The quadratic terms \(\vz_i^\top f(\mA)\vz_i\) are approximated by quadrature \(I_\idxLanczos^{(i)}\)  where the weights and nodes of the quadrature rule are computed via \(\idxLanczos\) iterations of the Lanczos algorithm. For the \(\log\)-determinant the following bound for the error incurred by Lanczos quadrature holds.



\begin{corollary}[Section 4.3 of \citet{Ubaru2017}]
	\label{cor:error_lanczos_log}
	Let \(\mA \in \Rnn\) be symmetric positive definite with condition number
	\(\kappa = \kappa(\mA)\). Then it holds that
	\begin{equation}
		\big\lvert \trSTE(\log(\mA)) - \trSLQ(\log(\mA))\big\rvert \leq
		K \left(\frac{\sqrt{2\kappa +1} - 1}{\sqrt{2\kappa +1}+1}\right)^{2\idxLanczos}
	\end{equation}
	where \(K = \frac{5\kappa \log(2(\kappa + 1))}{2 \sqrt{2\kappa +1}}\).
\end{corollary}

\section{Stochastic Trace Estimation}
\label{suppsec:stochastic-trace-estimation}

\begin{definition}[Convex Concentration Property \cite{Ledoux2001}]
	\label{def:convex-concentration-property}
	Let \(\vx \in \Rn\) be a random vector. We say \(\vx\) has the \emph{convex concentration property} (c.c.p.) with
	constant \(K \in \R\) if for every \(1\)-Lipschitz convex function \(\phi : \Rn \to \R\), we have
	\(\Exp[\abs{\phi(\vx)}] < \infty \) and for every \(t > 0\),
	\begin{equation*}
		\Pr[\abs{\phi(\vx) - \Exp[\phi(\vx)]} \geq t] \leq 2 \exp(-\frac{t^2}{K^2}).
	\end{equation*}
\end{definition}

Some common examples of random vectors having the c.c.p. are
\begin{itemize}
	\item random vectors with independent and almost surely bounded entries \(\abs{\vx_i} \leq
	      1\), where \(K=2 \sqrt{2}\) \cite{Samson2000};
	\item Gaussian random vectors \(\vx \sim \Normal(\vzero, \mSigma)\), where \(K^2 = 2
	      \norm{\mSigma}_2\) \cite{Kasiviswanathan2019}; and
	\item random vectors which are uniformly distributed on the sphere
	      \(\sqrt{n}\mathbb{S}^{n-1}\), where \(K=2\) \cite{Kasiviswanathan2019}.
\end{itemize}

\begin{remark}
	\label{rem:rademacher-ccp}
	Note, that since Rademacher random vectors have iid entries \(\{+1, -1\}\), they satisfy \(\norm{\tilde{\vz}_i}_2 = \sqrt{n}\). In particular, it holds that \(\sqrt{n}\vz_i = \tilde{\vz}_i\). Therefore the random vectors \(\sqrt{n}\vz_i\) and \(\vz'=\sqrt{n}(\vz_1, \dots, \vz_\idxrvs)^\top \in \R^{\idxrvs n}\) all have independent entries bounded by \(1\) and thus satisfy the convex concentration property with \(K=2\sqrt{2}\).
\end{remark}

\begin{theorem}[Hanson-Wright Inequality for Random Vectors with the Convex Concentration Property
		\cite{Adamczak2015}]
	\label{thm:hanson_wright_inequality}
	Let \(\mA \in \Rnn\) and \(\vx \in \Rn\) a zero-mean random vector with the convex concentration
	property with constant \(K\). Then for all \(t > 0\), it holds that
	\begin{equation}
		\Pr(\abs{\vx^\top \mA \vx - \Exp[\vx^\top \mA \vx]} \geq t) \leq 2 \exp(-c \min \bigg(\frac{t^2}{\norm{\mA}_F^2}, \frac{t}{\norm{\mA}_2}\bigg))
	\end{equation}
	where \(c =c(K)>0\) is a constant only dependent on the distribution of the random vectors.
\end{theorem}
\begin{proof}
	By Theorem 2.5 of \citet{Adamczak2015} we have
	\begin{equation*}
		\Pr(\abs{\vx^\top \mA \vx - \Exp[\vx^\top \mA \vx]} > t) \leq 2 \exp(-\frac{1}{c'} \min \bigg(\frac{t^2}{2K^4\norm{\mA}_F^2}, \frac{t}{K^2\norm{\mA}_2}\bigg))
	\end{equation*}
	where \(c'>0\) is a universal constant. Now it holds that
	\begin{equation*}
		\min \bigg(\frac{t^2}{2K^4\norm{\mA}_F^2}, \frac{t}{K^2\norm{\mA}_2}\bigg)  \geq \frac{1}{K^2 \max(2K^2, 1)} \min \bigg(\frac{t^2}{\norm{\mA}_F^2}, \frac{t}{\norm{\mA}_2}\bigg)
	\end{equation*}
	Choosing \(c = \frac{1}{c'K^2 \max(2K^2, 1)}\) concludes the proof.
\end{proof}

\begin{lemma}
	\label{lem:finite_sample_bound}
	Let \(\mA \in \Rnn\) and \(\idxrvs
	\in \N\). Consider \(\idxrvs\) random vectors \(\tilde{\vz}_i \in \Rn\) with zero mean and unit covariance, such that for \(\vz_i = \tilde{\vz}_i / \norm{\tilde{\vz}_i}_2\) the stacked random vector \(\sqrt{n}(\vz_1, \dots, \vz_\idxrvs)^\top \in \R^{\idxrvs n}\) has the convex concentration property.\footnote{See \Cref{rem:rademacher-ccp} for an explanation why this is satisfied for Rademacher random vectors.} Then there exists \(c_{\vz}>0\) such that if \(\idxrvs \geq
	c_{\vz}
	\log(\delta^{-1})\), then Hutchinson's trace estimator \(\trSTE\) satisfies
	\begin{equation*}
		\Pr(\abs{\trSTE(\mA) - \tr(\mA)} \leq \sqrt{c_{\vz} \log(\delta^{-1}) \idxrvs^{-1}} \norm{\mA}_F) \geq
		1-\delta.
	\end{equation*}
\end{lemma}

\begin{proof}
	Note that the proof strategy used here is the same as in \citet[Lemma~2]{Meyer2021} with a different assumption on the distribution of the random vectors. To begin, define
	\begin{equation*}
		\mA' = \begin{pmatrix} \mA & 0 & \hdots & 0\\ 0 & \mA & \ddots& \vdots\\ \vdots & \ddots & \ddots & 0 \\ 0 & \hdots & 0 & \mA \end{pmatrix} \in \R^{\idxrvs n \times \idxrvs n} \qquad \text{and}
		\qquad
		\vz' = \sqrt{n}\begin{pmatrix} \vz_1 \\ \vz_2 \\ \vdots \\ \vz_\idxrvs \end{pmatrix} \in \R^{\idxrvs n}.
	\end{equation*}
	By assumption the random vector \(\vz'\) has the convex concentration property and therefore
	\Cref{thm:hanson_wright_inequality} holds. We obtain
	\begin{equation}
		\Pr(\abs{(\vz')^\top \mA' \vz' - \Exp[(\vz')^\top \mA' \vz']} \geq t) \leq 2 \exp(-c \cdot \min\bigg(\frac{t^2}{\norm{\mA'}_F^2}, \frac{t}{\norm{\mA'}_2}\bigg)).
	\end{equation}
	Now, we have \((\vz')^\top \mA' \vz' = n\sum_{i = 1}^\idxrvs
	\vz_i^\top \mA \vz_i =\idxrvs \trSTE(\mA)\) and
	\begin{align*}
		\Exp[(\vz')^\top \mA' \vz'] & = n\sum_{i=1}^\idxrvs \Exp[\vz_i^\top \mA \vz_i] = n\idxrvs \Exp[\tr(\vz_i^\top \mA \vz_i)] = n\idxrvs \Exp[\tr(\mA \vz_i \vz_i^\top)] \\  &= n\idxrvs\tr(\mA \Exp[\vz_i \vz_i^\top]) = n\idxrvs \tr(\mA \Cov(\vz_i)) = \idxrvs \tr(\mA \Cov(\sqrt{n}\vz_i)) = \idxrvs \tr(\mA)
	\end{align*}
	Therefore by setting \(t =
	\sqrt{\frac{\log(2 \delta^{-1})}{c}\idxrvs}\norm{\mA}_F\), we obtain
	\begin{align*}
		\Pr\bigg(\idxrvs \abs{\trSTE(\mA) - \tr(\mA)} & \geq \sqrt{\frac{\log(2 \delta^{-1})}{c}\idxrvs}\norm{\mA}_F\bigg)                                                       \\ &\leq 2 \exp(-c \cdot \min\bigg( \frac{\log(2 \delta^{-1})}{c} \frac{\idxrvs \norm{\mA}_F^2}{\norm{\mA'}_F^2}, \sqrt{\frac{\log(2\delta^{-1})}{c}\idxrvs} \frac{\norm{\mA}_F}{\norm{\mA'}_2} \bigg))\\
		\intertext{Further, it holds that \(\norm{\mA'}_F^2 = \idxrvs \norm{\mA}_F^2\) and
			\(\norm{\mA'}_2 = \norm{\mA}_2\), thus we have}
		                                                    & = 2 \exp(- \min\bigg( \log(2 \delta^{-1}),  \sqrt{c\log(2\delta^{-1})\idxrvs} \frac{\norm{\mA}_F}{\norm{\mA}_2} \bigg)).
	\end{align*}
	Now assume \(\idxrvs \geq \frac{1}{c} \log(2 \delta^{-1})\). Then since
	\(\norm{\mA}_2\leq \norm{\mA}_F\), the minimum is given by
	\begin{equation*}
		\min \bigg(\log(2\delta^{-1}), \sqrt{c \log(2\delta^{-1}) \idxrvs} \frac{\norm{\mA}_F}{\norm{\mA}_2} \bigg) = \log(2 \delta^{-1}).
	\end{equation*}
	Further setting \(c_{\vz}=2c^{-1}\), it holds that \(\idxrvs \geq
	c_{\vz} \log(\delta^{-1}) =
	2\log(\delta^{-1})c^{-1} \geq \log(2\delta^{-1})c^{-1}\) since \(0 < \delta \leq
	\frac{1}{2}\). Combining the above we obtain
	\begin{align*}
		\Pr \left(\idxrvs\abs{\trSTE(\mA) - \tr(\mA)} \geq \sqrt{c_{\vz} \log(\delta^{-1})\idxrvs} \norm{\mA}_F \right) & \leq 2 \exp(-\log(2\delta^{-1})) = \delta,
		\intertext{which is equivalent to}
		\Pr \left(\abs{\trSTE(\mA) - \tr(\mA)} \leq \sqrt{c_{\vz} \log(\delta^{-1})\idxrvs^{-1}} \norm{\mA}_F \right)   & \geq 1- \delta.
	\end{align*}
	This proves the statement.
\end{proof}

\thmvarreducedtraceestimate*

\begin{proof}
	By assumption \(\abs{\trvarredux - \tr(f(\shat{\mK}))} = \abs{\trSTE(\mDelta_f) - \tr(\mDelta_f)}\). By \Cref{lem:finite_sample_bound}
	it
	holds with probability \(\geq 1-\delta\), that
	\begin{align*}
		\abs{\trSTE(\mDelta_f) - \tr(\mDelta_f)} & \leq \sqrt{c_{\vz} \log(\delta^{-1})\idxrvs^{-1}}\norm{\mDelta_f}_F                                                                                                                                   \\
		                                           & \leq \sqrt{c_{\vz} \log(\delta^{-1})} c_{\mDelta_f} \idxrvs^{-\frac{1}{2}}g(\idxrvs)\norm{f(\shat{\mK})}_F                                          & \text{Assumption \eqref{eqn:delta_frobenius_bound}.}      \\
												   &=\varepsilon_{\idxtxtSTE}(\delta, \idxrvs) \norm{f(\shat{\mK})}_F
	\end{align*}
	This concludes the proof.
\end{proof}

\begin{restatable}{corollary}{corvarreducedtraceestimate}
	\label{cor:variance-reduced-trace-estimate}
	Let \(\varepsilon \in (0, 1]\) be a desired error. If the conditions of \Cref{thm:variance-reduced-trace-estimate}
	hold and the number of random vectors \(\idxrvs\) satisfies
	\begin{equation}
		\label{eqn:varreducedhutch_rand_vectors}
		\idxrvs^{\frac{1}{2}}g(\idxrvs)^{-1} \geq
		C_1\varepsilon^{-1}
		\sqrt{\log(\delta^{-1})},
	\end{equation}
	then it holds that
	\begin{equation*}
		\boxed{
			\Pr(\abs{\trvarredux - \tr(\mA)} \leq \varepsilon \norm{\mA}_F) \geq 1-\delta.
		}
	\end{equation*}
\end{restatable}
\begin{proof}
	Follows from \Cref{thm:variance-reduced-trace-estimate} given \eqref{eqn:varreducedhutch_rand_vectors}.
\end{proof}

\section{Log-Determinant Estimation}
\label{suppsec:log-determinant-estimation}

\subsection{Approximation of a Matrix Function}

\begin{lemma}[Lipschitz Continuity]
	\label{lem:matrix_fun_lipschitz}
	Let \(\mA, \mB \in \Rnn\) be symmetric. Assume \(f:\Omega \to \R\) is globally Lipschitz continuous
	with Lipschitz constant \(L > 0\) on the combined spectrum \(\Omega = \lambda(\mA) \cup
	\lambda(\mB) \subset
	\R\), then there exists \(c_p > 0\) such that
	\begin{equation}
		\norm{f(\mA) - f(\mB)}_p \leq c_p L \norm{\mA - \mB}_p,
	\end{equation}
	where \(\norm{\cdot}_p\) denotes any matrix norm. In particular \(c_2 = 1\) and \(c_F =
	\sqrt{n}\).
\end{lemma}

\begin{proof}
	Since \(\mA, \mB\) are symmetric, they are normal. By \citet{Kittaneh1985}, it holds that
	\begin{equation*}
		\norm{f(\mA) - f(\mB)}_2 \leq L \norm{\mA - \mB}_2.
	\end{equation*}
	The result now follows by equivalence of norms on finite dimensional spaces. For the Frobenius norm
	we have \(\frac{1}{\sqrt{n}} \norm{\mM}_F \leq \norm{\mM}_2 \leq
	\norm{\mM}_F\), and therefore \(c_F = \sqrt{n}\).

\end{proof}

\begin{proposition}
	\label{prop:preconditioner_bound_delta_frobenius_bound}
	Let \(\shat{\mK} \in \Rnn\) be symmetric positive definite and assume \(f\) is analytic in a domain
	containing the spectrum \(\lambda(\shat{\mK})\). Let \(\{\shat{\mP}_\idxrvs\}_\idxrvs\) be a sequence of preconditioners with approximation quality \eqref{eqn:preconditioner-quality}. Then it holds that
	\begin{equation}
		\norm{f(\shat{\mK}) - f(\shat{\mP}_\idxrvs)}_F \leq c(n, \shat{\mK}, f) g(\idxrvs) \norm{f(\shat{\mK})}_F
	\end{equation}
	where \(c(n, \shat{\mK}, f) = \frac{L \norm{\shat{\mK}}_F}{c_{f(\lambda)}}\), \(L > 0\) is the Lipschitz constant of \(f\) and \(c_{f(\lambda)} = \max\{
	\min_i \abs{f(\lambda_i(\shat{\mK}))}, \frac{\max_i \abs{f(\lambda_i(\shat{\mK}))}}{\sqrt{n}}\}\).
\end{proposition}

\begin{proof}
	It holds that
	\begin{align*}
		\norm{f(\shat{\mK})}_F = \sqrt{\sum_{i=1}^n f(\lambda_i)^2} \geq \begin{cases} \sqrt{n \min_i f(\lambda_i)^2}= \sqrt{n} \min_i \abs{f(\lambda_i)} \\ \norm{f(\shat{\mK})}_2 = \sigma_{\max}(f(\shat{\mK})) = \sqrt{\lambda_{\max}(f(\shat{\mK})^2)} = \max_i \abs{f(\lambda_i)}\end{cases}
	\end{align*}
	and therefore \(\norm{f(\shat{\mK})}_F \geq \sqrt{n}
	c_{f(\lambda)}\). Since \(f\) is analytic and therefore Lipschitz, it holds that
	\begin{align*}
		\norm{f(\shat{\mK}) - f(\shat{\mP}_\idxrvs)}_F & \leq L \sqrt{n} \norm{\shat{\mK} - \shat{\mP}_\idxrvs}_F                                & \text{\Cref{lem:matrix_fun_lipschitz}}                                    \\
		                         & \leq L \sqrt{n} g(\idxrvs) \norm{\shat{\mK}}_F                               & \text{Preconditioner quality \eqref{eqn:preconditioner-quality}} \\
								 &= L \sqrt{n} g(\idxrvs)\frac{\norm{\shat{\mK}}_F}{\norm{f(\shat{\mK})}_F} \norm{f(\shat{\mK})}_F \\
								 &\leq
								 L \sqrt{n} g(\idxrvs)\frac{\norm{\shat{\mK}}_F}{\sqrt{n} c_{f(\lambda)}} \norm{f(\shat{\mK})}_F\\
		                         & \leq \frac{L \norm{\shat{\mK}}_F}{c_{f(\lambda)}} g(\idxrvs) \norm{f(\shat{\mK})}_F.
	\end{align*}
	This proves the claim.
\end{proof}

\subsection{Approximation of the Log-Determinant}
\label{suppsec:log-determinant}

\begin{lemma}[Decomposition of the \(\log\)-determinant]
	\label{lem:logdet-decomposition}
	For \(\shat{\mK}, \shat{\mP} \in \Rnn\) symmetric positive definite, it
	holds that
	\begin{align}
		\log \det(\shat{\mK}) & = \log\det(\shat{\mP}) + \tr(\log(\shat{\mK}) - \log(\shat{\mP}))                                    \\
		                      & = \log\det(\shat{\mP}) + \tr(\log(\shat{\mP}^{-\frac{1}{2}}\shat{\mK}\shat{\mP}^{-\frac{\top}{2}})).
	\end{align}
\end{lemma}
\begin{proof}
	Note that for symmetric positive definite matrices \(\mA, \mB\), the matrix logarithm satisfies the
	following
	\begin{align*}
		\log\det(\mA)     & = \tr(\log(\mA)),                  \\
		\tr(\log(\mA\mB)) & = \tr(\log(\mA)) + \tr(\log(\mB)), \\
		\log(\mA^{-1})    & = - \log(\mA).
	\end{align*}
	Using the above properties, we obtain
	\begin{align*}
		\log\det(\shat{\mK}) & = \tr(\log(\shat{\mP}\shat{\mP}^{-1}\shat{\mK}))                        \\
		                     & = \tr(\log(\shat{\mP})) - \tr(\log(\shat{\mP})) + \tr(\log(\shat{\mK})) \\
		                     & = \log\det(\shat{\mP}) + \tr(\log(\shat{\mK})-\log(\shat{\mP}))
	\end{align*}
	Now since \(\shat{\mP}^{-1}\shat{\mK}\) and
	\(\shat{\mP}^{-\frac{1}{2}}\shat{\mK}\shat{\mP}^
	{-\frac{\top}{2}}\) are similar, they have the same determinant. Therefore we have
	\begin{equation*}
		\tr(\log(\shat{\mK})-\log(\shat{\mP})) = \tr(\log(\shat{\mP}^{-1}\shat{\mK}))=\log\det(\shat{\mP}^{-\frac{1}{2}}\shat{\mK}\shat{\mP}^{-\frac{\top}{2}}).
	\end{equation*}
	This completes the proof.
\end{proof}

\convergencelogdet*

\begin{proof}
	Using the decomposition \eqref{eqn:logdet-decomposition}, we have
	\begin{align*}
		\abs{\trvarreduxlogdet - \log \det(\shat{\mK})} & = \abs{\trSLQ(\mDelta_{\log}) - \tr(\mDelta_{\log})}                                                                                                                                                            \\
		                                   & \leq \abs{\tr(\mDelta_{\log}) - \trSTE(\mDelta_{\log})} + \abs{\trSTE(\mDelta_{\log}) - \trSLQ(\mDelta_{\log})}                                                                                               \\
		                                   & = \underbrace{\abs{\tr(\log(\shat{\mK})) - (\tr(\log(\shat{\mP})) + \trSTE(\mDelta_{\log}))}}_{e_{\idxtxtSTE}} + \underbrace{\abs{\trSTE(\mDelta_{\log}) - \trSLQ(\mDelta_{\log})}}_{e_\text{Lanczos}}.
	\end{align*}
	Now the individual absolute errors are bounded as follows. By the error bound for stochastic trace
	estimation in \Cref{thm:variance-reduced-trace-estimate}, we have
	\begin{align*}
		e_{\idxtxtSTE}   & \leq \varepsilon_{\idxtxtSTE}(\delta, \idxrvs) \norm{\log(\shat{\mK})}_F=C_1 \sqrt{\log(\delta^{-1})} \idxrvs^{-\frac{1}{2}}g(\idxrvs) \norm{\log(\shat{\mK})}_F                                    \\
		\intertext{and by \Cref{cor:error_lanczos_log}, it follows that}
		e_\text{Lanczos} & \leq K \left(\frac{\sqrt{2\kappa +1} - 1}{\sqrt{2\kappa +1}+1}\right)^{2\idxLanczos} = K_1 \left(\frac{\sqrt{2\kappa +1} - 1}{\sqrt{2\kappa +1}+1}\right)^{2\idxLanczos} \norm{\log(\shat{\mK})}_F.
	\end{align*}
\end{proof}

\proberrorboundlogdet*

\begin{proof}
	By assumption \Cref{thm:variance-reduced-trace-estimate} is satisfied and therefore \(\varepsilon_{\idxtxtSTE} =
	\frac{\varepsilon}{2}\) with probability \(1-\delta\). Now for the error of Lanczos it holds by
	\Cref{thm:convergence-logdet} in combination with \Cref{lem:upper-bound-conv-frac}, that
	\begin{align*}
		\varepsilon_{\idxtxtLanczosSTE} & \leq K_1 \left(\frac{\sqrt{2\kappa +1} - 1}{\sqrt{2\kappa +1}+1}\right)^{2\idxLanczos}                                                              \\
		                             & \leq K_1 \exp(-\frac{4}{\sqrt{2\kappa+1}}\idxLanczos)      &\text{\Cref{lem:upper-bound-conv-frac}}                                                                                          \\
		                             & \leq K_1 \exp(-\frac{\sqrt{3\kappa}}{\sqrt{2\kappa+1}}\log(2K_1\varepsilon^{-1}))      & \text{By assumption \eqref{eqn:num_lanczos_steps_logdet}.} \\
		                             & \leq K_1 \exp(-\log(2K_1\varepsilon^{-1}))                                                                                                          \\
		                             & = \frac{\varepsilon}{2}
	\end{align*}
	The result now follows by \Cref{thm:convergence-logdet}.
\end{proof}

\subsection{Approximation of the Derivative of the Log-Determinant}
\label{suppsec:derivative-log-determinant}

\paragraph{Computation of \(\tr(\shat{\mP}^{-1} \pdv{\shat{\mP}}{\evtheta})\)}

\Cref{alg:log_marginal_likelihood} and \Cref{alg:derivative_log_marginal_likelihood} primarily rely on matrix-vector
multiplication, except for computation of $\trprecondinvderiv = \tr(\shat{\mP}^{-1} \pdv{\shat{\mP}}{\evtheta})$. Efficient computation of this term
depends on the structure of $\shat{\mP}^{-1}$.
If $\shat{\mP}$ is the pivoted-Cholesky preconditioner, or any other
diagonal-plus-low-rank preconditioner $\sigma^2 \mI + \mL_\idxrvs \mL_\idxrvs^\top$, we can rewrite this term using the
matrix inversion lemma
\begin{align}
	\tr\left( \shat{\mP}^{-1} \pdv{\shat{\mP}}{\evtheta} \right)
	 & =
	\sigma^{-2} \tr\left( \pdv{\shat{\mP}}{\evtheta} \right)
	-
	\sigma^{-2} \tr\left( \mL_\idxrvs \left( \sigma^{2} \mI + \mL_\idxrvs^\top \mL_\idxrvs \right)^{-1}
	\mL_\idxrvs^\top \pdv{\shat{\mP}}{\evtheta} \right)
	\nonumber \\
	 & =
	\sigma^{-2} \sum_{i=1}^n \pdv{\shat{\mP}_{ii}}{\evtheta}
	-
	\sigma^{-2} \left( \left( \mL_\idxrvs \left( \sigma^{2} \mI + \mL_\idxrvs^\top \mL_\idxrvs \right)^{-1} \right)
	\circ \left( \pdv{\shat{\mP}}{\evtheta} \mL_\idxrvs \right) \right) \vone,
	\label{eqn:preconditioner_trace}
\end{align}
where $\circ$ denotes elementwise multiplication.
The second term requires $\idxrvs$ matrix-vector multiplies with
$\pdv{\shat{\mP}}{\evtheta}$ and $\mathcal{O}(n \idxrvs^2)$ additional work.
The first term is simply the derivative of the kernel diagonal which will take
$\mathcal{O}(n)$ time.
We note that similar efficient procedures exists for other types of preconditioners,
such as when $\shat{\mP}^{-1}$ has banded structure.

\convergencelogdetbackward*

\begin{proof}
	Using the decomposition \eqref{eqn:logdet-decomposition_deriv}, we have
	\begin{align*}
		\abs{\trvarreduxinvderiv - \tr(\shat{\mK}^{-1} \pdv{\shat{\mK}}{\evtheta})} & = \abs{\trSCG(\mDelta_{\mathrm{inv}\partial}) - \tr(\mDelta_{\mathrm{inv}\partial})}                                                                                                                                                            \\
		                                   & \leq \abs{\tr(\mDelta_{\mathrm{inv}\partial}) - \trSTE(\mDelta_{\mathrm{inv}\partial})} + \abs{\trSTE(\mDelta_{\mathrm{inv}\partial}) - \trSCG(\mDelta_{\mathrm{inv}\partial})}                                                                                               \\
		                                   & = \underbrace{\abs{\tr(\shat{\mK}^{-1} \pdv{\shat{\mK}}{\evtheta}) - (\tr(\shat{\mP}^{-1} \pdv{\shat{\mP}}{\evtheta}) + \trSTE(\mDelta_{\mathrm{inv}\partial}))}}_{e_{\idxtxtSTE}} + \underbrace{\abs{\trSTE(\mDelta_{\mathrm{inv}\partial}) - \trSCG(\mDelta_{\mathrm{inv}\partial})}}_{e_\text{CG}}.
	\end{align*}
	Now the individual absolute errors are bounded as follows. By the error bound for stochastic trace
	estimation in \Cref{thm:variance-reduced-trace-estimate}, we have
	\begin{align*}
		e_{\idxtxtSTE}   & \leq \varepsilon_{\idxtxtSTE}(\delta, \idxrvs) \norm{\shat{\mK}^{-1} \pdv{\shat{\mK}}{\evtheta}}_F=C_1 \sqrt{\log(\delta^{-1})} \idxrvs^{-\frac{1}{2}}g(\idxrvs) \norm{\shat{\mK}^{-1} \pdv{\shat{\mK}}{\evtheta}}_F.                                    \\
	\end{align*}
	Now, let \(\vw_i = \shat{\mK}^{-1} \pdv{\shat{\mK}}{\evtheta}\vz_i\), \(\tilde{\vw}_i = \shat{\mP}^{-1} \pdv{\shat{\mP}}{\evtheta}\vz_i\) and \(\vw_{\idxCG, i} \approx \vw_i\) be the solution computed via preconditioned CG with \(\idxCG\) iterations. Then we have by \Cref{thm:convergence_rate_cg}, that
	\begin{align*}
		e_\text{CG} & = \abs{\frac{n}{\idxrvs}\sum_{i=1}^\idxrvs \vz_i^\top(\vw_{\idxCG,i} - \tilde{\vw}_i - (\vw_i - \tilde{\vw}_i))}\\
		&\leq \frac{n}{\idxrvs}\sum_{i=1}^{\idxrvs} \underbrace{\norm{\vz_i}_2}_{=1} \norm{\vw_{\idxCG, i} - \vw_i}_2 \\
		&\leq \frac{n}{\idxrvs}\sum_{i=1}^{\idxrvs} 2 \sqrt{\kappa(\shat{\mK})} \left( \frac{\sqrt{\kappa} - 1}{\sqrt{\kappa} + 1} \right)^\idxCG \norm{\vw_{0,i} - \vw_i}_2 &\text{CG convergence by \Cref{thm:convergence_rate_cg}.}\\
		&\leq 2 n\sqrt{\kappa(\shat{\mK})} \left( \frac{\sqrt{\kappa} - 1}{\sqrt{\kappa} + 1} \right)^\idxCG \norm{\vw_i}_2 &\text{CG initialized at \(\vw_{0,i}=\vzero\) or better.}\\
		&\leq K_2 \left( \frac{\sqrt{\kappa} - 1}{\sqrt{\kappa} + 1} \right)^\idxCG \norm{\shat{\mK}^{-1} \pdv{\shat{\mK}}{\evtheta}\vz_i}_F\\
		&\leq K_2 \left( \frac{\sqrt{\kappa} - 1}{\sqrt{\kappa} + 1} \right)^\idxCG \norm{\shat{\mK}^{-1} \pdv{\shat{\mK}}{\evtheta}}_F
	\end{align*}
	This completes the argument.
\end{proof}

\proberrorboundlogdetbackward*

\begin{proof}
	By assumption \Cref{thm:variance-reduced-trace-estimate} is satisfied and therefore \(\varepsilon_{\idxtxtSTE} =
	\frac{\varepsilon}{2}\) with probability \(1-\delta\). Now for the error of CG, it holds by
	\Cref{thm:convergence-logdet-backward} in combination with \Cref{lem:upper-bound-conv-frac}, that
	\begin{align*}
		\varepsilon_{\idxtxtCG} & \leq K_2 \left(\frac{\sqrt{\kappa} - 1}{\sqrt{\kappa}+1}\right)^{\idxCG}\\
		&\leq K_2 \exp(-\frac{2\idxCG}{\sqrt{\kappa}}) &\text{\Cref{lem:upper-bound-conv-frac}}\\
		&\leq K_2 \exp(-\log(2K_2\varepsilon^{-1}))&\text{Assumption \eqref{eqn:num-cg-steps-logdet-backward}}\\
		&\leq \frac{\varepsilon}{2}.
	\end{align*}
	The result now follows by \Cref{thm:convergence-logdet-backward}.
\end{proof}

\section{GP Hyperparameter Optimization}
\label{suppsec:gp-hyperparameter-optimization}

\subsection{Approximation of the Log-Marginal Likelihood}
\label{suppsec:log-marginal-likelihood}

\convergencelogmarginallikelihood*

\begin{proof}
	It holds by assumption that 
	\begin{align*}
		\abs{\eta - \logmarglik } & = \frac{1}{2}\abs{\vy^\top \vu_{\idxCG} + \trvarreduxlogdet - (\vy^\top \shat{\mK}^{-1}\vy +  \log\det(\shat{\mK}))}                      \\
		                                            & \leq \frac{1}{2} \big(\underbrace{\abs{\vy^\top \vu_{\idxCG} - \vy^\top \vu}}_{e_{\idxtxtCG}} + \underbrace{\abs{\trvarreduxlogdet - \log\det(\shat{\mK})}}_{e_{\text{SLQ}}} \big).
	\end{align*}
	For the error of CG when solving \(\shat{\mK}\vu = \vy\), we have by \Cref{thm:convergence_rate_cg}
	\begin{align*}
		e_{\idxtxtCG}&= \abs{\vy^\top \vu_{\idxCG} - \vy^\top \vu} \leq \norm{\vy}_2 \norm{\vu_{\idxCG} - \vu}_2    \leq \norm{\vy}_2 2 \sqrt{\kappa(\shat{\mK})} \left( \frac{\sqrt{\kappa} - 1}{\sqrt{\kappa} + 1} \right)^\idxCG \norm{\vu_0 - \vu}_2 = 2 \varepsilon_{\idxtxtCG},\\
		\intertext{and for the absolute error in the \(\log\)-determinant estimate via preconditioned stochastic Lanczos quadrature, we obtain by \Cref{thm:convergence-logdet}, that}
		e_{\text{SLQ}}&\leq (\varepsilon_{\idxtxtLanczosSTE} +
		\varepsilon_{\idxtxtSTE}) \norm{\log(\shat{\mK})}_F.
	\end{align*}
	This proves the statement.
\end{proof}

\begin{restatable}{corollary}{proberrorboundlogmarginallikelihood}
	\label{cor:error-bound-log-marginal-likelihood}
	 Assume the conditions of \Cref{thm:convergence-log-marginal-likelihood} hold. If the number of random vectors \(\idxrvs\) satisfies \eqref{eqn:varreduced-ste-error} with \(\varepsilon_{\idxtxtSTE} = \varepsilon\), and we run \(\idxCG \geq \max(\idxCG_{\idxtxtCG}, \idxLanczos_{\mathrm{Lanczos}})\) iterations of CG and Lanczos, where
	\begin{align}
		\label{eqn:num-cg-steps-logmarginallikelihood}
		\textstyle \idxCG_{\idxtxtCG} &\geq \frac{1}{2}\sqrt{\kappa}\log(2K_3 \varepsilon^{-1}),\\
		\label{eqn:num-lanczos-steps-logmarginallikelihood}
		\textstyle \idxCG_{\mathrm{Lanczos}} &\geq \frac{\sqrt{3}}{4}\sqrt{\kappa}\log\big(K_1\varepsilon^{-1}\big),
	\end{align}
	then it holds that
	\begin{equation*}
		\boxed{
			\textstyle
			\Pr \left(\abs{\eta - \logmarglik } \leq \varepsilon (1 + \norm{\log(\shat{\mK})}_F)\right) \geq
			1-\delta.
		}
	\end{equation*}
\end{restatable}

\begin{proof}
	We begin with the error of CG, it holds by
	\Cref{thm:convergence-log-marginal-likelihood} in combination with \Cref{lem:upper-bound-conv-frac}, that
	\begin{align*}
		\varepsilon_{\idxtxtCG} & \leq K_3 \left(\frac{\sqrt{\kappa} - 1}{\sqrt{\kappa}+1}\right)^{\idxCG}\\
		&\leq K_3 \exp(-\frac{2\idxCG}{\sqrt{\kappa}}) &\text{\Cref{lem:upper-bound-conv-frac}}\\
		&\leq K_3 \exp(-\log(2K_3\varepsilon^{-1}))&\text{Assumption \eqref{eqn:num-cg-steps-logmarginallikelihood}}\\
		&= \frac{\varepsilon}{2}.
	\end{align*}

	Now for the error of the estimate of the \(\log\)-determinant. By assumption \Cref{thm:variance-reduced-trace-estimate} is satisfied and therefore \(\varepsilon_{\idxtxtSTE} =
	\varepsilon\) with probability \(1-\delta\). For the error of Lanczos, it holds by \Cref{thm:convergence-log-marginal-likelihood}, that
	\begin{align*}
		\varepsilon_{\idxtxtLanczosSTE} & \leq K_1 \left(\frac{\sqrt{2\kappa +1} - 1}{\sqrt{2\kappa +1}+1}\right)^{2\idxLanczos}                                                              \\
		                             & \leq K_1 \exp(-\frac{4}{\sqrt{2\kappa+1}}\idxLanczos)      &\text{\Cref{lem:upper-bound-conv-frac}}                                                                                          \\
		                             & \leq K_1 \exp(-\frac{\sqrt{3\kappa}}{\sqrt{2\kappa+1}}\log(K_1\varepsilon^{-1}))      & \text{Assumption \eqref{eqn:num-lanczos-steps-logmarginallikelihood}.} \\
		                             & \leq K_1 \exp(-\log(K_1\varepsilon^{-1}))                                                                                                          \\
		                             & = \varepsilon
	\end{align*}
	The result now follows by \Cref{thm:convergence-log-marginal-likelihood}.
\end{proof}

\subsection{Approximation of the Derivative of the Log-Marginal Likelihood}
\label{suppsec:derivative-log-marginal-likelihood}

\convergencederivative*

\begin{proof}
	It holds that
	\begin{align*}
		\abs{\phi - \pdv{}{\evtheta}\logmarglik } &= \frac{1}{2} \abs{\vu_\idxCG^\top \pdv{\shat{\mK}}{\evtheta}\vu_\idxCG - \trvarreduxinvderiv - \bigg(\vy^\top \shat{\mK}^{-1}
		{ \pdv{\shat{\mK}}{\evtheta} } \shat{\mK}^{-1}\vy - 
		\tr(\shat{\mK}^{-1} {\pdv{\shat{\mK}}{\evtheta}})\bigg)}\\
		&\leq \bigg(\underbrace{\abs{\vu_\idxCG^\top \pdv{\shat{\mK}}{\evtheta}\vu_\idxCG - \vu^\top \pdv{\shat{\mK}}{\evtheta}\vu}}_{e_{\idxtxtCG}} + \underbrace{\abs{\trvarreduxinvderiv - \tr(\shat{\mK}^{-1} {\pdv{\shat{\mK}}{\evtheta}})}}_{e_{\text{CGSTE}}} \bigg)
	\end{align*}
	Now by \Cref{thm:convergence-logdet-backward}, we have
	\begin{align*}
		e_{\text{CGSTE}} \leq (\varepsilon_{\idxtxtCGSTE} + \varepsilon_{\idxtxtSTE}) \norm{\shat{\mK}^{-1}\pdv{\shat{\mK}}{\evtheta}}_F.
	\end{align*}
	For the absolute error of the quadratic term, it holds that
	\begin{align*}
		e_{\idxtxtCG} &= \abs{\norm{\vu}_{\pdv{\shat{\mK}}{\evtheta}}^2 - \norm{\vu_\idxCG - \vu + \vu}_{\pdv{\shat{\mK}}{\evtheta}}^2}\\
		&\leq \abs{\norm{\vu}_{\pdv{\shat{\mK}}{\evtheta}}^2 - (\norm{\vu_\idxCG - \vu}_{\pdv{\shat{\mK}}{\evtheta}} + \norm{\vu}_{\pdv{\shat{\mK}}{\evtheta}})^2}\\
		&= \norm{\vu_\idxCG - \vu}_{\pdv{\shat{\mK}}{\evtheta}} + 2\norm{\vu_\idxCG - \vu}_{\pdv{\shat{\mK}}{\evtheta}}\norm{\vu}_{\pdv{\shat{\mK}}{\evtheta}}\\
		&\leq \norm{\pdv{\shat{\mK}}{\evtheta}}_2 (\norm{\vu_\idxCG - \vu}_2^2 + 2 \norm{\vu_\idxCG - \vu}_2 \norm{\vu}_2)\\
		&\leq \norm{\pdv{\shat{\mK}}{\evtheta}}_2 g(\norm{\vu_\idxCG - \vu}_2)(1+2 \norm{\vu}_2)
		\intertext{for \(g(t) = \max(t, t^2)\). Now it holds by \Cref{thm:convergence_rate_cg} and monotoncity of \(g\), that}
		&\leq \norm{\pdv{\shat{\mK}}{\evtheta}}_2 g\left(2 \sqrt{\kappa(\shat{\mK})}\left(\frac{\sqrt{\kappa}-1}{\sqrt{\kappa}+1}\right)^\idxCG\norm{\vu_0 - \vu}_2\right)(1+2 \norm{\vu}_2)\\
		\intertext{Since for \(a\leq 1\), it holds that \(g(at) \leq ag(t)\) and for \(a > 1: g(at)\leq a^2g(t)\), we have}
		&\leq \norm{\pdv{\shat{\mK}}{\evtheta}}_2 4 \kappa(\shat{\mK}) \left(\frac{\sqrt{\kappa}-1}{\sqrt{\kappa}+1}\right)^\idxCG g(\norm{\vu_0 - \vu}_2)(1+2 \norm{\vu}_2)\\
		&\leq \norm{\pdv{\shat{\mK}}{\evtheta}}_2 4 \kappa(\shat{\mK}) \left(\frac{\sqrt{\kappa}-1}{\sqrt{\kappa}+1}\right)^\idxCG 3\max(\norm{\vu}_2, \norm{\vu}_2^3)\\
		&= 2 K_4 \left(\frac{\sqrt{\kappa}-1}{\sqrt{\kappa}+1}\right)^\idxCG
	\end{align*}
	where we used that CG was initialized at \(\vu_0 = \vzero\) or better.
\end{proof}

\begin{restatable}{corollary}{proberrorboundderivative}
	\label{thm:error-bound-derivative}
	 Assume the conditions of \Cref{thm:convergence-derivative} hold. If the number of random vectors \(\idxrvs\) satisfies \eqref{eqn:varreduced-ste-error} with \(\varepsilon_{\idxtxtSTE} = \varepsilon\), and we run \(\idxCG \geq \max(\idxCG_{\idxtxtCG}, \idxCG_{\idxtxtCGSTE})\) iterations of CG and Lanczos, where
	\begin{align}
		\label{eqn:num-cg-steps-derivative}
		\textstyle \idxCG_{\idxtxtCG} &\geq \frac{1}{2}\sqrt{\kappa}\log(2K_4 \varepsilon^{-1}),\\
		\label{eqn:num-cg2-steps-derivative}
		\textstyle \idxCG_{\idxtxtCGSTE} &\geq \frac{1}{2}\sqrt{\kappa}\log(K_2 \varepsilon^{-1}),
	\end{align}
	then it holds that
	\begin{equation*}
		\boxed{
			\textstyle
			\Pr \left(\abs{\phi - \pdv{}{\evtheta}\logmarglik(\vtheta) } \leq	\varepsilon (1 + \norm{\mK^{-1}\pdv{\mK}{\evtheta}}_F)\right) \geq
			1-\delta.
		}
	\end{equation*}
\end{restatable}

\begin{proof}
	We begin with the error of CG's estimate of \(\vy^\top \shat{\mK}^{-1}
	{ \pdv{\shat{\mK}}{\evtheta} } \shat{\mK}^{-1}\vy\), it holds by
	\Cref{thm:convergence-derivative} in combination with \Cref{lem:upper-bound-conv-frac}, that
	\begin{align*}
		\varepsilon_{\idxtxtCG} & \leq K_4 \left(\frac{\sqrt{\kappa} - 1}{\sqrt{\kappa}+1}\right)^{\idxCG}\\
		&\leq K_4 \exp(-\frac{2\idxCG}{\sqrt{\kappa}}) &\text{\Cref{lem:upper-bound-conv-frac}}\\
		&\leq K_4 \exp(-\log(2K_4\varepsilon^{-1}))&\text{Assumption \eqref{eqn:num-cg-steps-derivative}}\\
		&= \frac{\varepsilon}{2}.
	\end{align*}

	Now for the error of the stochastic trace estimator. By assumption \Cref{thm:variance-reduced-trace-estimate} is satisfied and therefore \(\varepsilon_{\idxtxtSTE} =
	\varepsilon\) with probability \(1-\delta\). For the error of CG used in the stochastic trace estimate, we obtain by \Cref{thm:convergence-derivative}
	\begin{align*}
		\varepsilon_{\idxtxtCGSTE} & \leq K_2 \left(\frac{\sqrt{\kappa} - 1}{\sqrt{\kappa}+1}\right)^{\idxCG}\\
		&\leq K_2 \exp(-\frac{2\idxCG}{\sqrt{\kappa}}) &\text{\Cref{lem:upper-bound-conv-frac}}\\
		&\leq K_2 \exp(-\log(K_2\varepsilon^{-1}))&\text{Assumption \eqref{eqn:num-cg2-steps-derivative}}\\
		&= \varepsilon.
	\end{align*}
	The result now follows by \Cref{thm:convergence-derivative}.
\end{proof}

\section{Preconditioning}
\label{suppsec:preconditioning}

\begin{lemma}[Condition Number and Preconditioner Quality]
	\label{lem:condition-number-preconditioner-quality}
	Let \(\shat{\mK}, \shat{\mP}_{\idxrvs} \in \Rnn\) symmetric positive-definite such
	that \eqref{eqn:preconditioner-quality} holds and assume that there exists \(c>0\) such that \(c \geq
	\max(\lambda_{\min}(\shat{\mP}), \lambda_{\min}(\shat{\mK}))\) for all \(\idxrvs\). Then it holds that
	\begin{equation}
		\kappa = \kappa(\shat{\mP}_{\idxrvs}^{-\frac{1}{2}}\shat{\mK}\shat{\mP}_{\idxrvs}^{-\frac{1}{2}}) \leq (1 + \bigO(g(\idxrvs))\norm{\shat{\mK}}_F)^2
	\end{equation}
\end{lemma}
\begin{proof}
	Part of the strategy for this proof is adapted from \citet{Gardner2018}. First note, that the
	matrices \(\shat{\mP}^{-1}\shat{\mK}\),
	\(\shat{\mK}\shat{\mP}^{-1}\) and
	\(\shat{\mP}_{\idxrvs}^{-\frac{1}{2}}\shat{\mK}\shat{\mP}_{\idxrvs}^
	{-\frac{1}{2}}\) are similar, and thus have the same eigenvalues. Now, we have:
	\begin{align*}
		\kappa(\shat{\mP}_{\idxrvs}^{-\frac{1}{2}}\shat{\mK}\shat{\mP}_{\idxrvs}^{-\frac{1}{2}}) & = \frac{\lambda_{\max}(\shat{\mP}^{-1}\shat{\mK})}{\lambda_{\min}(\shat{\mK}\shat{\mP}^{-1})} = \norm{\shat{\mP}^{-1}\shat{\mK}}_2 \norm{\shat{\mP}\shat{\mK}^{-1}}_2 \\
		                                                                                         & = \norm{\shat{\mP}^{-1}(\shat{\mK} - \shat{\mP} + \shat{\mP})}_2 \norm{(\shat{\mP} - \shat{\mK} +
			\shat{\mK})\shat{\mK}^{-1}}_2                                                                                                                                                                                                                                    \\
		                                                                                         & = \norm{1 + \shat{\mP}^{-1}(\shat{\mK} - \shat{\mP})}_2 \norm{1 - \shat{\mK}^{-1}(\shat{\mK} - \shat{\mP})}_2                                                         \\
		\intertext{Applying Cauchy-Schwarz and the triangle inequality, we obtain:}
		                                                                                         & \leq (1 + \norm{\shat{\mP}^{-1}}_2 \norm{\shat{\mK} - \shat{\mP}}_2) (1 + \norm{\shat{\mK}^{-1}}_2 \norm{\shat{\mK} - \shat{\mP}}_2)                                  \\
		                                                                                         & \leq (1 + c \norm{\shat{\mK} - \shat{\mP}}_F)^2                                                                                                                       \\
		                                                                                         & \leq (1 + \bigO(g(\idxrvs))\norm{\shat{\mK}}_F)^2
	\end{align*}
\end{proof}

Note since typically \(\shat{\mP} = \sigma^2 \mI+ \mP\) with
\(\lambda_{\min}(\mP) \approx 0\) and \(\lambda_{\min}(\shat{\mK}) \leq \sigma^2 +
\lambda_{\min}(\mK)\), where \(\lambda_{\min}(\mK)\) small for most kernels, usually \(c
\approx \sigma^2\).

\subsection{Additive Kernels}

\begin{lemma}[Additive Kernels]
	\label{lem:additive-kernel-approx-quality}
	Let \(k(\vx, \vy) = \sum_{j=1}^d k_j(\vx_j, \vy_j)\) be an additive kernel and
	\(\{(\shat{\mP}_\idxrvs)_j\}_{j=1}^d\) a set of preconditioners indexed by \(\idxrvs\), such that for all
	\(j =
	1, \dots, d\), we have
	\begin{equation}
		\norm{\shat{\mK}_j - (\shat{\mP}_\idxrvs)_j}_F \leq c_j g(\idxrvs) \norm{\shat{\mK}_j}_F.
	\end{equation}
	Then it holds for \(\shat{\mP}_\idxrvs = \sum_{j=1}^d (\shat{\mP}_\idxrvs)_j\) and \(c=\max_j c_j\)
	that
	\begin{equation}
		\norm{\shat{\mK} - \shat{\mP}_\idxrvs}_F \leq c d g(\idxrvs) \norm{\shat{\mK}}_F
	\end{equation}
\end{lemma}

\begin{proof}
	It holds by assumption, that
	\begin{align*}
		\norm{\shat{\mK} - \shat{\mP}_\idxrvs}_F & \leq \sum_{j=1}^d \norm{\shat{\mK}_j - (\shat{\mP}_\idxrvs)_j}_F                             & \text{Cauchy-Schwarz}                                                   \\
		                           & \leq \sum_{j=1}^d c_j g(\idxrvs) \norm{\shat{\mK}_j}_F                                                                                                          \\
		                           & \leq \max_j c_j g(\idxrvs) \sum_{j=1}^d \sqrt{\sum_{i=1}^n \lambda_i(\shat{\mK}_j)^2}                                                                           \\
		                           & \leq c g(\idxrvs) \sum_{j=1}^d \sqrt{\sum_{i=1}^n \lambda_i(\shat{\mK})^2}            & \mA, \mB \text{ spd } \implies \lambda_i(\mA) \leq \lambda_i(\mA + \mB) \\
		                           & \leq c g(\idxrvs) d \norm{\shat{\mK}}_F
	\end{align*}
\end{proof}

\subsection{Kernels with a Uniformly Converging Approximation}

\begin{lemma}[Preconditioner Quality from Uniform Convergence]
	\label{lem:uniform-convergence-approx-quality}
	Let \(\Omega \subset \R^d\) be the data domain and \(k(\cdot, \cdot)\) a positive-definite kernel
	such that for all \(\vx \in \Omega\) it holds that \(k(\vx, \vx) \leq o^2\). Let
	\(P_\idxrvs(\cdot, \cdot)\) be a kernel approximation, such that a uniform convergence
	bound of the form
	\begin{equation}
		\label{eqn:uniform-convergence-bound-kernelapprox}
		\sup_{\vx, \vy \in \Omega} \abs{k(\vx, \vy) - P_\idxrvs(\vx, \vy)} \leq g(\idxrvs) c_{\text{unif}}(d,
		\Omega, k)
	\end{equation}
	holds. Then it holds for the preconditioner \(\shat{\mP}_\idxrvs = \sigma^2 \mI + P_\idxrvs(\mX, \mX)\), that
	\begin{equation}
		\norm{\shat{\mK} - \shat{\mP}_\idxrvs}_F \leq g(\idxrvs) c(d, \Omega, k, n) \norm{\shat{\mK}}_F,
	\end{equation}
	where \(c(d, \Omega, k, n) = \frac{\sqrt{n}}{o^2}c_{\text{unif}}(d, \Omega, k)\).
\end{lemma}
\begin{proof}
	It holds that
	\begin{align*}
		\norm{\shat{\mK} - \shat{\mP}_\idxrvs}_F &=\norm{\mK - \mP_\idxrvs}_F = \sqrt{\sum_{i,j = 1}^n (k(\vx_i, \vx_j) - P_\idxrvs(\vx_i, \vx_j))^2}                                                                                                                        \\
		                           & \leq \sqrt{\sum_{i,j = 1}^n (g(\idxrvs) c_{\text{unif}}(d, \Omega, k))^2}                                & \text{uniform convergence bound \eqref{eqn:uniform-convergence-bound-kernelapprox}} \\
		                           & = n g(\idxrvs) c_{\text{unif}}(d, \Omega, k)                                                                                                                                                   \\
		                           & = g(\idxrvs) \sqrt{n} \bigg(\sum_{i=1}^n \frac{k(\vx_i, \vx_i)^2}{k(\vx_i, \vx_i)^2}\bigg)^{\frac{1}{2}} & \text{\(k\) bounded}                                                                \\
		                           & \leq g(\idxrvs) \frac{\sqrt{n}}{o^2} \bigg(\sum_{i=1}^n \mK_{ij}^2 \bigg)^{\frac{1}{2}}                                                                                                        \\
		                           & = g(\idxrvs) c(d, \Omega, k, n) \norm{\mK}_F\\
								   &\leq g(\idxrvs) c(d, \Omega, k, n) \norm{\shat{\mK}}_F
	\end{align*}
\end{proof}

\subsection{Cholesky Decomposition}


\begin{proposition}[Cholesky Approximation Quality]
	\label{prop:cholesky-approx-quality}
	Let \(k(\vx, \vy) = k(\norm{\vx - \vy})\) be a stationary kernel
	with output scale \(o^2 = k(0) > 0\).
	Assume the kernel matrix spectrum decays at least exponentially, i.e. \(\lambda_i(\mK) \leq
	c\exp(-bi)\) for \(c >0\) and \(b > \log(4)\). Then the
	Cholesky preconditioner \(\shat{\mP}_\idxrvs = \sigma^2 \mI + \mL_\idxrvs \mL_\idxrvs^\top\)satisfies
	\begin{equation}
		\lVert \shat{\mK} - \shat{\mP}_\idxrvs \rVert_F \leq \frac{c
		\sqrt{n}}{o^{2}}
		\exp(-b' \idxrvs)\norm{\mK}_F
	\end{equation}
	where \(b' = b - \log(4) > 0\).
\end{proposition}
\begin{proof}
	First note that since \(k(\cdot, \cdot)\) is a positive definite kernel, the choice \(o^2 =
	k(0)\) is no restriction. Now, it holds that
	\begin{align*}
		\lambda_i            & \leq c\exp(-b i) = c \exp(-(b' + \log(4))i) = c \exp(-b' i)4^{-i} \\
		\iff 4^{i} \lambda_i & \leq c \exp(-b'i).
	\end{align*}
	Therefore by Theorem 3.2 of \citet{Harbrecht2012}, we have \(\tr(\mK - \mL_\idxrvs \mL_\idxrvs^\top ) \leq cn \exp(-b' \idxrvs)\).	Now it holds since \(\mK - \mL_\idxrvs \mL_\idxrvs^\top\) positive definite, that
	\begin{align*}
		\lVert \shat{\mK} - \shat{\mP}_\idxrvs \rVert_F &= \lVert \mK - \mL_\idxrvs \mL_\idxrvs^\top \rVert_F\leq \tr(\mK - \mL_\idxrvs \mL_\idxrvs^\top )\leq cn \exp(-b' \idxrvs),
	\end{align*}
	and with \(\mK_{ii} = o^2\) that
	\begin{equation*}
		n = \sqrt{n} \bigg(\sum_{i=1}^n \frac{\mK_{ii}^2}{o^4}\bigg)^{\frac{1}{2}} \leq
		\frac{\sqrt{n}}{o^2} \bigg(\sum_{i,j=1}^n \mK_{ij}^2\bigg)^{\frac{1}{2}} = \frac{\sqrt{n}}{o^2}
		\norm{\mK}_F.
	\end{equation*}
	This concludes the argument.
\end{proof}

\subsection{Quadrature Fourier Features (QFF)}

\begin{proposition}[QFF Approximation Quality]
	\label{prop:qff-approx-quality}
	Assume \(\Omega = [0, 1]^d\), \(k\) a kernel with Fourier transform
	\(p(\omega) =
	\exp(- \frac{1}{2}\sum_{j=1}^d \omega_j^2 \gamma_j^2)\) such that Assumption 1 of \citet{Mutny2018} is satisfied and let
	\(\shat{\mP}_\idxrvs = \sigma^2\mI + \mP_\idxrvs\), where \(\mP_\idxrvs\) is the QFF approximated kernel matrix. Let \(\idxrvs^{\frac{1}{d}} > \frac{2}{\gamma^2}\), then for \(b=\frac{1}{2}(\log(4)-1)\),
	it holds that
	\begin{equation}
		\norm{\shat{\mK} - \shat{\mP}_\idxrvs}_F \leq c(d,n,k)\exp(-b\idxrvs^{\frac{1}{d}}) \norm{\shat{\mK}}_F.
	\end{equation}
\end{proposition}

\begin{proof}
	By Theorem 1 of \citet{Mutny2018}, replacing \(\idxrvs = (2\bar{m})^d\) it
	holds
	that
	\begin{align*}
		\sup_{\vx, \vy \in \Omega} \abs{k(\vx, \vy) - P_\idxrvs(\vx, \vy)} & \leq d 2^{d -1} \sqrt{\frac{\pi}{2}} \bigg(\frac{e}{2\gamma^2 \idxrvs^{\frac{1}{d}}}\bigg)^{\frac{1}{2}\idxrvs^{\frac{1}{d}}}                                \\
		                                                                   & \leq c(d) \bigg(\frac{e}{4}\bigg)^{\frac{1}{2}\idxrvs^{\frac{1}{d}}}                                                          & \idxrvs^{\frac{1}{d}} > 2 \gamma^{-2} \\
		                                                                   & =c(d)\exp(\frac{1}{2}\idxrvs^{\frac{1}{d}}-\log(4)\frac{1}{2}\idxrvs^{\frac{1}{d}})\\
																		   &= c(d) \exp(-\frac{1}{2}(\log(4)-1)\idxrvs^{\frac{1}{d}})                                                                                                  \\
		                                                                   & = c(d) \exp(-b \idxrvs^{\frac{1}{d}})
	\end{align*}
	Now by \Cref{lem:uniform-convergence-approx-quality}, we have
	\begin{equation}
		\norm{\shat{\mK} - \shat{\mP}_\idxrvs}_F \leq c(d,n)\exp(-b\idxrvs^{\frac{1}{d}}) \norm{\shat{\mK}}_F,
	\end{equation}
	where \(c(d,n) = \sqrt{n} c(d)\) by Assumption 1 of \citet{Mutny2018}, which assumes \(k(\vx, \vy) \leq
	1\).
\end{proof}

\begin{proposition}[General QFF Approximation Quality]
	\label{prop:general-qff-approx-quality}
	Assume \(\Omega = [0, 1]^d\), \(k\) a kernel with Fourier transform
	\(p(\omega)\) such that Assumption 1 of \citet{Mutny2018} is satisfied and
	\(f_\delta(\phi) = p(\cot(\phi))\frac{\cos(\delta \cot(\phi))}{\sin(\phi)^2}\) is
	\((s-1)\)-times absolutely continuous. Let \(\shat{\mP}_\idxrvs = \sigma^2\mI + \mP_\idxrvs\), where \(\mP_\idxrvs\) is the QFF approximated kernel matrix. Then it holds that
	\begin{equation}
		\norm{\shat{\mK} - \shat{\mP}_\idxrvs}_F \leq c(d, \Omega, k, n)\idxrvs^{-\frac{s + 1}{d}} \norm{\shat{\mK}}_F.
	\end{equation}
\end{proposition}

\begin{proof}
	By Theorem 4 of \cite{Mutny2018}, replacing \(\idxrvs = (2\bar{m})^d\), it
	holds that
	\begin{align*}
		\sup_{\vx, \vy \in \Omega} \abs{k(\vx, \vy) - P_\idxrvs(\vx, \vy)} & \leq d 2^{d-1}\frac{(s + 2)^{s+1}}{s!} \max_{\delta \in \Omega} \operatorname{TV}(f_\delta^{(s)}) 2^{s+1} \idxrvs^{- \frac{s+1}{d}} \\
		                                                                   & = c(d, \Omega, k) \idxrvs^{- \frac{s+1}{d}}
	\end{align*}
	Now by \Cref{lem:uniform-convergence-approx-quality}, we have
	\begin{equation*}
		\norm{\shat{\mK} - \shat{\mP}_\idxrvs}_F \leq c(d, \Omega, k, n)\idxrvs^{-\frac{s + 1}{d}} \norm{\shat{\mK}}_F.
	\end{equation*}
\end{proof}

\begin{remark}[Modified Mat\'ern Kernel]
	\Cref{prop:general-qff-approx-quality} is satisfied for example for the modified
	Mat\'ern(\(\nu\)) kernel, defined via its spectral density
	\begin{equation}
		p(\omega) = \prod_{j=1}^d \frac{1}{(1 + \gamma^2_j \omega_j^2)^{\nu + \frac{1}{2}}}.
	\end{equation}
	See the appendix of \citet{Mutny2018} for a detailed definition and motivation.
\end{remark}

\subsection{Truncated Singular Value Decomposition}

\begin{proposition}[Truncated SVD Approximation Quality]
	\label{prop:truncated-svd-approx-quality}
	Let \(\mK\) be a kernel matrix and \(\mP_\idxrvs = \mV_\idxrvs \mLambda_\idxrvs \mV_\idxrvs\) its truncated singular value decomposition consisting of the  eigenvectors of the largest \(\idxrvs\) eigenvalues. Then it holds for \(\shat{\mP} = \sigma^2\mI + \mP_\idxrvs\), that
	\begin{equation}
		\norm{\shat{\mK} - \shat{\mP}_\idxrvs}_F \leq c(n)\idxrvs^{-\frac{1}{2}} \norm{\shat{\mK}}_F.
	\end{equation}
	where \(c(n) = \sqrt{n}\).
\end{proposition}
\begin{proof}
	Since the optimal rank-\(\idxrvs\) approximation in Frobenius norm is given by the truncated SVD \cite{Eckart1936}, we have for \(\lambda_1(\mK) \geq \dots \geq \lambda_n(\mK)\), that
	\begin{align*}
		\norm{\shat{\mK} - \shat{\mP}_\idxrvs}_F^2 &= \norm{\mK - \mP_\idxrvs}_F^2 = \sum_{i=\idxrvs +1}^n \lambda_i(\mK)^2\\
		\intertext{Now since \(\lambda_{\idxrvs +1}(\mK) \leq \frac{1}{\idxrvs} \sum_{i=1}^{\idxrvs}\lambda_i(\mK) \leq \frac{1}{\idxrvs}\tr(\mK)\), we have}
		&\leq \lambda_{\idxrvs+1}(\mK) \sum_{i=\idxrvs+1}^n \lambda_i(\mK) \leq \frac{1}{\idxrvs}\tr(\mK)^2 = \frac{1}{\idxrvs} \norm{\vlambda}_1^2 \leq \frac{n}{\idxrvs} \norm{\vlambda}_2^2 = \frac{n}{\idxrvs} \norm{\mK}_F^2 \leq \frac{n}{\idxrvs} \norm{\shat{\mK}}_F^2
	\end{align*}
	where \(\vlambda = (\lambda_1(\mK), \dots, \lambda_n(\mK))^\top \in \Rn\).
\end{proof}

\subsection{Randomized Singular Value Decomposition}

\begin{proposition}[Randomized SVD Approximation Quality]
	\label{prop:randomized-svd-approx-quality}
	Let \(\mK\) be a kernel matrix and \(\mP_\idxrvs = \mH_\idxrvs \mH_\idxrvs^\top \mK\) its randomized singular value decomposition constructed via the \textsc{LinearTimeSVD} algorithm \cite{Drineas2006} with \(s \in \N\) samples drawn according to probabilities \(\{p_i\}_{i=1}^n\). Then for \(\shat{\mP} = \sigma^2\mI + \mP_\idxrvs\), it holds with probability \(1-\delta\), that
	\begin{equation}
		\norm{\shat{\mK} - \shat{\mP}_\idxrvs}_F \leq c(n)(\idxrvs^{-\frac{1}{2}} + \bigO(\idxrvs^{\frac{1}{4}} s^{-\frac{1}{4}})) \norm{\shat{\mK}}_F,
	\end{equation}
	where \(c(n) = \sqrt{n}\).
\end{proposition}

\begin{proof}
	By \citet[Theorem~4]{Drineas2006} it holds with probability \(1-\delta\), that
	\begin{align*}
		\norm{\shat{\mK} - \shat{\mP}_\idxrvs}_F^2 &= \norm{{\mK} - {\mP}_\idxrvs}_F^2 \leq \norm{\mK - \mK_\idxrvs}_F^2 + c(p_i, \delta) \idxrvs^{\frac{1}{2}} s^{-\frac{1}{2}} \norm{\mK}_F^2\\
		\intertext{By the same argument as in the proof of \Cref{prop:truncated-svd-approx-quality} for the error  of the optimal rank-\(\idxrvs\) approximation \(\norm{\mK - \mK_\idxrvs}_F^2\), we obtain}
		&\leq \frac{n}{\idxrvs} \norm{\mK}_F^2 + c(p_i, \delta)\idxrvs^{\frac{1}{2}} s^{-\frac{1}{2}} \norm{\mK}_F^2\\
		&\leq n (\idxrvs^{-1} + \bigO(\idxrvs^{\frac{1}{2}} s^{-\frac{1}{2}})) \norm{\shat{\mK}}_F^2
	\end{align*}
	This completes the proof.
\end{proof}

\subsection{Randomized Nyström Method}

\begin{proposition}[Randomized Nyström Approximation Quality]
	\label{prop:nystroem-approx-quality}
	Let \(\mK\) be a kernel matrix and \(\mP_\idxrvs = \mC \mW_\idxrvs^+ \mC^\top\) its randomized Nyström approximation constructed via Algorithm 3 of \citet{Drineas2005} with \(s \in \N\) columns drawn according to probabilities \(p_i=\frac{\mK_{ii}^2}{\sum_{i=1}^n \mK_{ii}^2}\). Then for \(\shat{\mP} = \sigma^2\mI + \mP_\idxrvs\), it holds with probability \(1-\delta\), that
	\begin{equation}
		\norm{\shat{\mK} - \shat{\mP}_\idxrvs}_F \leq c(n)(\idxrvs^{-\frac{1}{2}} + \bigO(\idxrvs^{\frac{1}{4}} s^{-\frac{1}{4}})) \norm{\shat{\mK}}_F,
	\end{equation}
	where \(c(n) = \sqrt{n}\).
\end{proposition}

\begin{proof}
	By \citet[Theorem~3]{Drineas2005} it holds with probability \(1-\delta\), that
	\begin{align*}
		\norm{\shat{\mK} - \shat{\mP}_\idxrvs}_F &= \norm{{\mK} - {\mP}_\idxrvs}_F \leq \norm{\mK - \mK_\idxrvs}_F + c(\delta) \idxrvs^{\frac{1}{4}} s^{-\frac{1}{4}} \sum_{i=1}^n \mK_{ii}^2\\
		\intertext{By the same argument as in the proof of \Cref{prop:truncated-svd-approx-quality} for the error  of the optimal rank-\(\idxrvs\) approximation \(\norm{\mK - \mK_\idxrvs}_F\), we obtain}
		&\leq n^{\frac{1}{2}}\idxrvs^{-\frac{1}{2}} \norm{\mK}_F^2 + c(\delta) \idxrvs^{\frac{1}{4}} s^{-\frac{1}{4}} \sum_{i=1}^n \mK_{ii}^2\\
		&\leq n^{\frac{1}{2}}\idxrvs^{-\frac{1}{2}} \norm{\mK}_F^2 + c(\delta) \idxrvs^{\frac{1}{4}} s^{-\frac{1}{4}} \sqrt{\sum_{i=1}^n \mK_{ii}^2} \norm{\mK}_F\\
		&\leq n^{\frac{1}{2}} (\idxrvs^{-\frac{1}{2}} + \bigO(\idxrvs^{\frac{1}{4}} s^{-\frac{1}{4}})) \norm{\shat{\mK}}_F
	\end{align*}
	This completes the proof.
\end{proof}

\subsection{Random Fourier Features (RFF)}

\begin{proposition}[RFF Approximation Quality]
	\label{prop:rff-approx-quality}
	Let \(k(\vx, \vy) = k(\vx - \vy)\) be a positive-definite kernel with compact data domain
	\(\Omega \subset \R^d\). Let \(\mP_\idxrvs = \mZ_\idxrvs \mZ_\idxrvs^\top\) be the random Fourier feature approximation \cite{Rahimi2007}, where \(\mZ_\idxrvs \in \R^{n \times \idxrvs}\). Then for \(\shat{\mP}_\idxrvs = \sigma^2 \mI + \mP_\idxrvs\) it holds with probability \(1-\delta\), that
	\begin{equation}
		\norm{\shat{\mK} - \shat{\mP}_\idxrvs}_F \leq c(d, \Omega, k, \delta, n)
		\idxrvs^{-\frac{1}{2}}\norm{\shat{\mK}}_F.
	\end{equation}
\end{proposition}
\begin{proof}
	By Theorem 1 of \citet{Sriperumbudur2015} with probability \(1 - \delta\), we obtain the uniform
	convergence bound
	\begin{equation}
		\sup_{\vx, \vy \in \Omega} \abs{k(\vx, \vy) - P_\idxrvs(\vx, \vy)} \leq 
		c_{\text{unif}}(d,\Omega, k, \delta) \idxrvs^{-\frac{1}{2}}.
	\end{equation}
	Now, applying \Cref{lem:uniform-convergence-approx-quality} completes the proof.
\end{proof}

\section{Technical Results}
\label{suppsec:technical-results}

\begin{lemma}
	\label{lem:upper-bound-conv-frac}
	Let \(x \in \R\) such that \(x>1\), then it holds that
	\begin{equation}
		\frac{x-1}{x+1} \leq \exp(-\frac{2}{x}).
	\end{equation}
\end{lemma}
\begin{proof}
	By \citet[Section~3.6.18]{Mitrinovic1970} it holds for \(y > 0\), that
	\begin{equation*}
		\frac{\log(y+1)}{y} \geq \frac{2}{2+y}
	\end{equation*}
	Substituting \(y = \frac{x+1}{x-1}-1\) we obtain
	\begin{align*}
		\log(\frac{x+1}{x-1}) \geq \frac{2(\frac{x+1}{x-1}-1)}{1+\frac{x+1}{x-1}} = \frac{2 (x + 1 - x + 1)}{2x} = \frac{2}{x} \\
		\iff \frac{x+1}{x-1} \geq \exp(\frac{2}{x})
	\end{align*}
	This proves the claim.
\end{proof}

\section{Additional Experimental Results}
\label{suppsec:additional-experimental-results}

\subsection{Synthetic Data}

We report bias and variance of the stochastic estimators for the \(\log\)-maginal likelihood and
its derivatives for the exponentiated quadratic, Mat\'ern(\(\frac{3}{2}\)) and
rational quadratic kernel on a synthetic dataset of size \(n=10,\!000\) with varying dimensionality
\(d \in \{1,2,3\}\) in \Cref{tab:synthetic-data}. Using \(\idxrvs=128\) random samples with a
preconditioner of the same size bias and variance are reduced by several orders of magnitude across
different kernels. Note, that the variance reduction tends to decline with dimensionality, even
though this is not necessarily universal across kernels. We show the bias and variance reduction
for increasing number of random samples, respectively preconditioner size in
\Cref{fig:synthetic-data-detailed}.

\begin{table}
	\caption{\emph{Bias and variance reduction for different kernels.} Bias and variance
	of the stochastic estimators for the \(\log\)-marginal likelihood and its derivative(s) computed
	for synthetic data (\(n=10,\!000\), \(\sigma^2=10^{-2}\)) with \(25\) repetitions.}
	\label{tab:synthetic-data}
	\small
	\centering
	\begin{tabular}{cS[table-format=1]S[table-format=1]S[table-format=1e-2]S[table-format=1e-2]S[table-format=1e-2]S[table-format=1e-2]S[table-format=1e-2]S[table-format=1e-2]S[table-format=1e-2]S[table-format=1e-2]S[table-format=1e-2]S[table-format=1e-2]S[table-format=1e-2]S[table-format=1e-2]S[table-format=1e-2]S[table-format=1e-2]S[table-format=1e-2]S[table-format=1e-2]S[table-format=1e-2]S[table-format=1e-2]}
\toprule
        &   &     & \multicolumn{2}{c}{$\logmarglik$} & \multicolumn{2}{c}{$\partial \logmarglik / \partial o$} & \multicolumn{2}{c}{$\partial \logmarglik / \partial l$} & \multicolumn{2}{c}{$\partial \logmarglik / \partial \sigma$} \\
        &   &     & {Bias} & {Var.} &                  {Bias} & {Var.} &                     {Bias} & {Var.} &                       {Bias} & {Var.} \\
Kernel & {$d$} & {Prec. Qual.} &        &        &                         &        &                            &        &                              &        \\
\midrule
Mat{\' e}rn$(3/2)$ & 1 & 0   &  3e-04 &  3e-08 &                   9e-04 &  1e-09 &                      2e-03 &  8e-09 &                        2e-03 &  2e-09 \\
        &   & 128 &  9e-06 &  4e-11 &                   4e-06 &  8e-12 &                      1e-05 &  7e-11 &                        7e-06 &  2e-11 \\
        & 2 & 0   &  3e-01 &  3e-03 &                   4e-01 &  4e-03 &                      1e+00 &  3e-02 &                        9e-01 &  3e-02 \\
        &   & 128 &  3e-04 &  5e-08 &                   7e-05 &  3e-09 &                      2e-04 &  3e-08 &                        1e-04 &  7e-09 \\
        & 3 & 0   &  3e-01 &  6e-04 &                   4e-01 &  7e-04 &                      1e+00 &  6e-03 &                        8e-01 &  4e-03 \\
        &   & 128 &  7e-03 &  7e-07 &                   2e-02 &  3e-08 &                      5e-02 &  4e-07 &                        3e-02 &  8e-08 \\
RBF & 1 & 0   &  1e-04 &  1e-08 &                   1e-05 &  4e-11 &                      1e-04 &  3e-09 &                        2e-05 &  1e-10 \\
        &   & 128 &  5e-08 &  1e-15 &                   3e-08 &  4e-16 &                      7e-07 &  2e-13 &                        4e-08 &  1e-15 \\
        & 2 & 0   &  3e-03 &  7e-08 &                   5e-03 &  2e-09 &                      8e-02 &  1e-07 &                        1e-02 &  2e-09 \\
        &   & 128 &  1e-06 &  1e-12 &                   5e-07 &  2e-13 &                      1e-05 &  1e-10 &                        8e-07 &  5e-13 \\
        & 3 & 0   &  1e-01 &  2e-07 &                   2e-01 &  4e-08 &                      2e+00 &  9e-07 &                        4e-01 &  8e-09 \\
        &   & 128 &  3e-04 &  5e-08 &                   5e-05 &  2e-09 &                      7e-04 &  3e-07 &                        8e-05 &  4e-09 \\
RatQuad & 1 & 0   &  1e-04 &  1e-08 &                   1e-05 &  1e-10 &                      1e-04 &  5e-09 &                        2e-05 &  3e-10 \\
        &   & 128 &  3e-07 &  6e-14 &                   2e-07 &  2e-14 &                      2e-06 &  4e-12 &                        2e-07 &  4e-14 \\
        & 2 & 0   &  3e-02 &  2e-07 &                   4e-02 &  3e-08 &                      3e-01 &  2e-07 &                        1e-01 &  7e-09 \\
        &   & 128 &  8e-05 &  3e-09 &                   2e-05 &  2e-10 &                      2e-04 &  1e-08 &                        3e-05 &  5e-10 \\
        & 3 & 0   &  2e-01 &  2e-04 &                   3e-01 &  3e-04 &                      2e+00 &  1e-02 &                        8e-01 &  2e-03 \\
        &   & 128 &  4e-04 &  1e-07 &                   2e-04 &  1e-08 &                      4e-03 &  4e-07 &                        4e-04 &  4e-08 \\
\bottomrule
\end{tabular}

\end{table}


\begin{figure}
	\centering
	\begin{subfigure}[b]{0.9\textwidth}
		\centering
		\includegraphics[width=\textwidth]{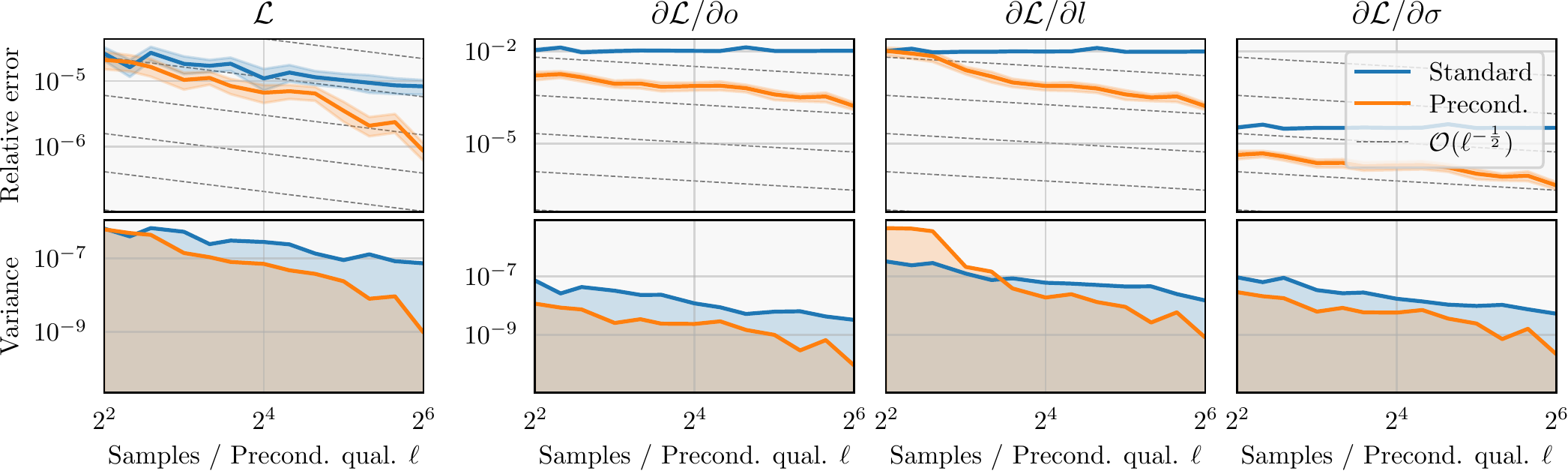}
		\caption{Mat{\'e}rn(\(3/2\)) (\(d=1\)).}
	\end{subfigure}\\%
	~
	\begin{subfigure}[b]{0.9\textwidth}
		\centering
		\includegraphics[width=\textwidth]{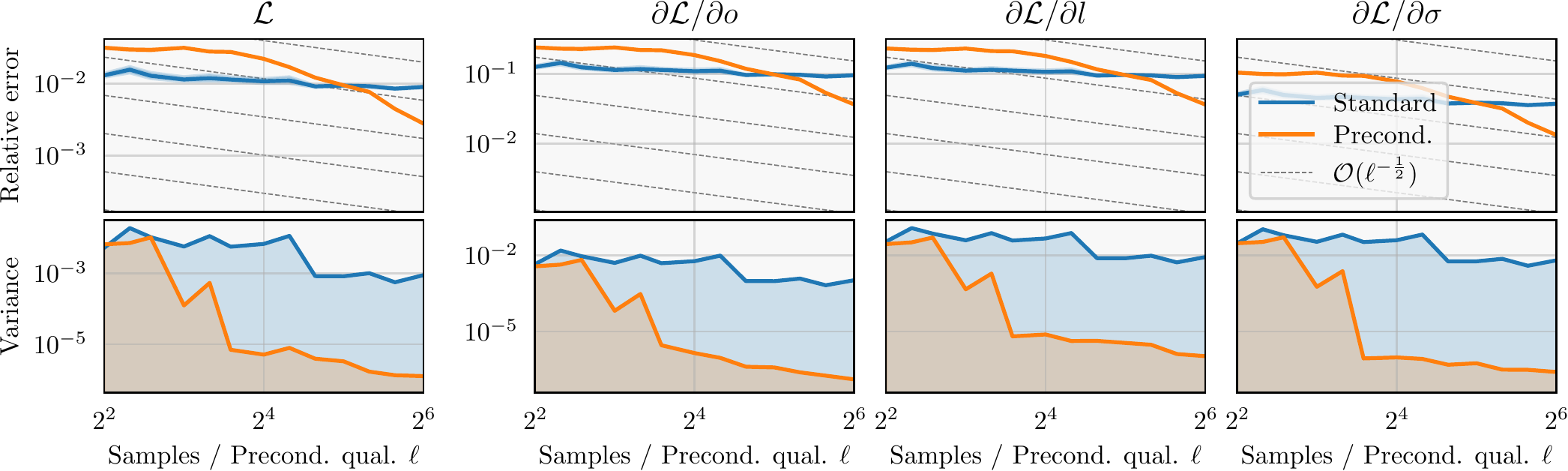}
		\caption{Mat{\'e}rn(\(3/2\)) (\(d=3\)).}
	\end{subfigure}\\%
	~
	\begin{subfigure}[b]{0.9\textwidth}
		\centering
		\includegraphics[width=\textwidth]{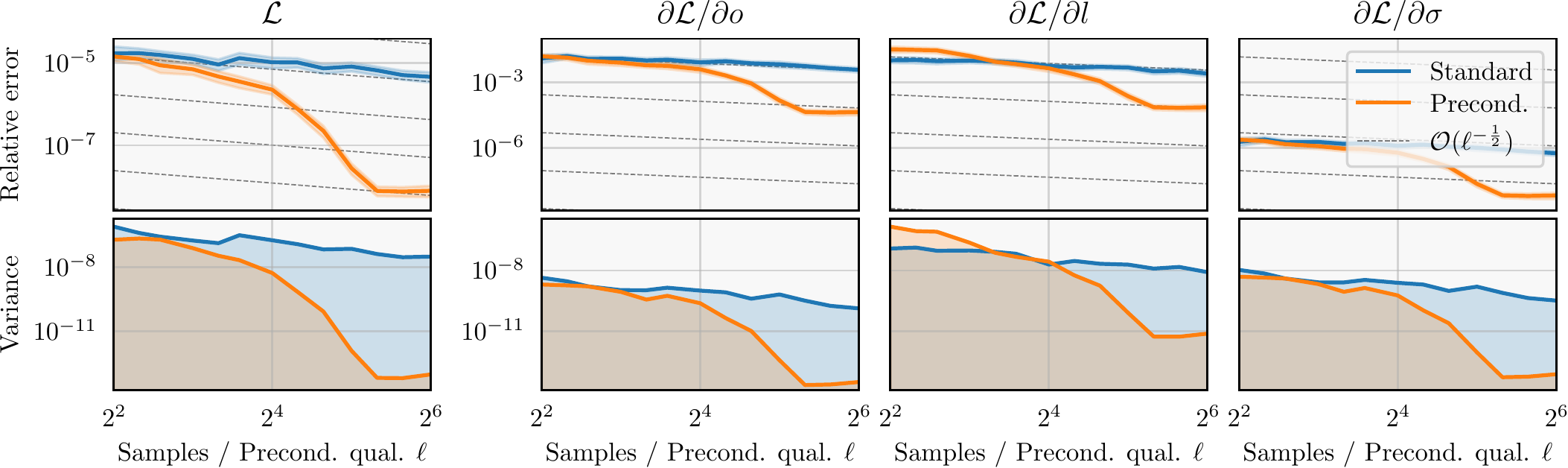}
		\caption{RatQuad (\(d=1\)).}
	\end{subfigure}\\%
	~
	\begin{subfigure}[b]{0.9\textwidth}
		\centering
		\includegraphics[width=\textwidth]{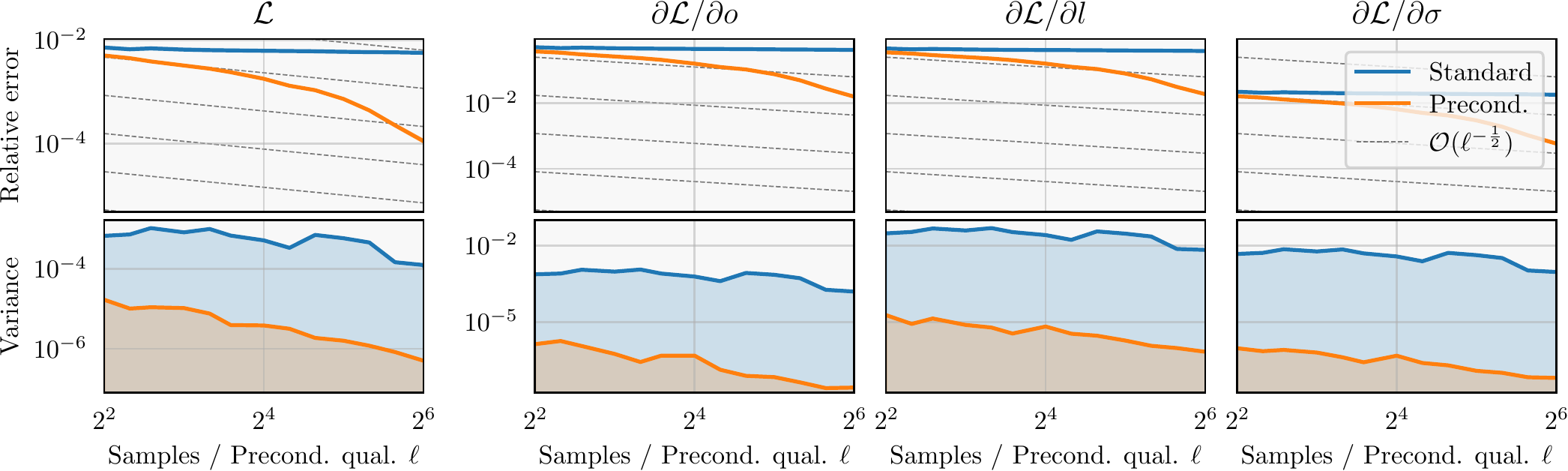}
		\caption{RatQuad (\(d=3\)).}
	\end{subfigure}%
	\caption{\emph{Bias and variance decrease on synthetic datasets for different kernels.} Relative error and variance of the stochastic estimators of the
		\(\log\)-marginal likelihood and its derivative for increasing number of random
		vectors, equivalently preconditioner size. Experiments were performed for different kernels on a synthetic dataset of size \(n=10,\!000\) with dimension \(d \in \{1, 2, 3\}\). Plots show mean and 95\% confidence intervals for the relative error computed over 25
		repetitions.}
	\label{fig:synthetic-data-detailed}

\end{figure}

\subsection{UCI Datasets}

For the experiments we conducted on UCI datasets, we report the full experimental results with
their deviation across 10 runs in \Cref{tab:uci-results-detailed}. Test errors with and without
preconditioning did not differ by more than two standard deviations. However, model evaluations of
the optimizer were significantly reduced when using a preconditioner of size 500, leading to
substantial speedup. Note, that the experiment on the ``3DRoad'' dataset was only carried out once
due to the prohibitive runtime without preconditioning.

We used the L-BFGS optimizer in our experiments due to its favorable convergence properties. As an
ablation experiment we compared to the Adam optimizer as sometimes used for its robustness to
noise, when using stochastic approximations of the \(\log\)-marginal likelihood
\cite{Gardner2018,Wang2019}. We find that with preconditioning optimization with L-BFGS significantly
outperformed optimization with Adam, both in terms of training and test error, except for the
``KEGGdir'' dataset (cf. \Cref{tab:uci-results} and \Cref{tab:uci-results-adam}). Additionally,
L-BFGS converged faster across all experiments. This shows that variance reduction via
preconditioning makes the use of second-order optimizers not only possible, but preferred for GP
hyperparameter optimization when using stochastic approximations.


\begin{table}
	\caption{\emph{Hyperparameter optimization using Adam.} GP regression using a Mat\'ern\((\frac{3}{2})\) kernel and
		pivoted Cholesky preconditioner of size \(500\) with \(\idxrvs=50\) random samples.
		Hyperparameters
		were optimized with Adam for at most 20 steps using early stopping via a validation set.}
	\label{tab:uci-results-adam}
	\small
	\centering
	\renewrobustcmd{\bfseries}{\fontseries{b}\selectfont}
	\sisetup{detect-weight=true,detect-inline-weight=math}
	\begin{tabular}{cS[table-format=6]S[table-format=2]S[table-format=3]cS[table-format=-1.4,round-mode=places,round-precision=4]S[table-format=-1.4,round-mode=places,round-precision=4]S[table-format=-1.4,round-mode=places,round-precision=4]S[table-format=3]}
\toprule
{Dataset} &  {$n$} &  {$d$} &  {Prec. Size} &  {$-\logmarglik_{\mathrm{train}} \downarrow$} &  {$-\logmarglik_{\mathrm{test}} \downarrow$} &  {RMSE $\downarrow$} &  {Runtime (s)} \\
\midrule
Elevators &  12449 &     18 &           500 &                                       0.4803 &                                      0.4593 &              0.36840 &            109 \\
     Bike &  13034 &     17 &           500 &                                       0.2265 &                                      0.3473 &              0.23000 &             64 \\
   Kin40k &  30000 &      8 &           500 &                                       0.4392 &                                     -0.1200 &              0.09821 &            159 \\
  Protein &  34297 &      9 &           500 &                                       0.9438 &                                      0.9319 &              0.56810 &             92 \\
  KEGGdir &  36620 &     20 &           500 &                                      -1.0070 &                                     -1.0390 &              0.08100 &            239 \\
\bottomrule
\end{tabular}

\end{table}

\begin{sidewaystable}
	\caption{\emph{Hyperparameter optimization on UCI datasets.}
		GP regression using a Mat\'ern\((\frac{3}{2})\) kernel and
		pivoted Cholesky preconditioner of size \(500\) with \(\idxrvs=50\) random samples.
		Hyperparameters
		were optimized with L-BFGS for at most 20 steps using early stopping via a validation set. All
		results, but ``3DRoad'', are averaged over 10 runs.}
	\label{tab:uci-results-detailed}
	\footnotesize
	\centering
	\renewrobustcmd{\bfseries}{\fontseries{b}\selectfont}
	\sisetup{detect-weight=true,detect-inline-weight=math}
	\begin{tabular}{cS[table-format=6]S[table-format=2]S[table-format=3]S[table-format=2]S[table-format=3.1,round-mode=places,round-precision=1]S[table-format=2.1,round-mode=places,round-precision=1]S[table-format=5.1,round-mode=places,round-precision=1]S[table-format=4.1,round-mode=places,round-precision=1]S[table-format=2.1,round-mode=places,round-precision=1]S[table-format=1.1,round-mode=places,round-precision=1]S[table-format=-1.4,round-mode=places,round-precision=4]S[table-format=-1.4,round-mode=places,round-precision=4]S[table-format=-1.4,round-mode=places,round-precision=4]S[table-format=-1.4,round-mode=places,round-precision=4]S[table-format=-1.4,round-mode=places,round-precision=4]S[table-format=-1.4,round-mode=places,round-precision=4]}
\toprule
       &        &    &     &    & \multicolumn{2}{c}{{Model evals.}} & \multicolumn{2}{c}{{Runtime (s)}} & \multicolumn{2}{c}{{Speedup}} & \multicolumn{2}{c}{{$-\logmarglik_{\mathrm{train}} \downarrow$}} & \multicolumn{2}{c}{{$-\logmarglik_{\mathrm{test}} \downarrow$}} & \multicolumn{2}{c}{{RMSE $\downarrow$}} \\
       &        &    &     &    &         {mean} &      {std} &        {mean} &       {std} &     {mean} &     {std} &                                      {mean} &     {std} &                                     {mean} &     {std} &              {mean} &     {std} \\
Dataset & {$n$} & {$d$} & {Prec. Qual.} & {Opt. Steps} &                &            &               &             &            &           &                                             &           &                                            &           &                     &           \\
\midrule
Elevators & 12449  & 18 & 0   & 19 &      42.181818 &   7.534165 &     53.000000 &    7.733046 &   1.000000 &  0.000000 &                                    0.464722 &  0.003496 &                                   0.402140 &  0.010196 &            0.348366 &  0.005177 \\
       &        &    & 500 & 19 &      36.363636 &   1.206045 &     39.181818 &    0.750757 &   1.353316 &  0.202710 &                                    0.437725 &  0.002954 &                                   0.402184 &  0.009727 &            0.348241 &  0.005313 \\
Bike & 13034  & 17 & 0   & 19 &      32.272727 &   1.348400 &     30.636364 &    1.120065 &   1.000000 &  0.000000 &                                   -0.997622 &  0.011435 &                                  -0.993428 &  0.013740 &            0.044620 &  0.003969 \\
       &        &    & 500 & 19 &      31.363636 &   2.062655 &     37.090909 &    1.513575 &   0.826877 &  0.037716 &                                   -0.998517 &  0.018617 &                                  -0.987725 &  0.017996 &            0.045370 &  0.003166 \\
Kin40k & 30000  & 8  & 0   & 8  &      19.000000 &   4.449719 &    186.545455 &   67.962289 &   1.000000 &  0.000000 &                                   -0.333929 &  0.001791 &                                  -0.314085 &  0.001641 &            0.092942 &  0.001439 \\
       &        &    & 500 & 6  &      15.000000 &   0.447214 &     44.636364 &    1.286291 &   2.710737 &  0.076995 &                                   -0.433196 &  0.005531 &                                  -0.313514 &  0.003783 &            0.094906 &  0.001457 \\
Protein & 34297  & 9  & 0   & 15 &     124.181818 &   0.404520 &    892.636364 &   19.143003 &   1.000000 &  0.000000 &                                    0.996320 &  0.003125 &                                   0.886924 &  0.008142 &            0.572161 &  0.006525 \\
       &        &    & 500 & 7  &      17.818182 &   0.603023 &     42.545455 &    4.987256 &   4.179380 &  0.386223 &                                    0.927287 &  0.004397 &                                   0.883540 &  0.005298 &            0.557747 &  0.007926 \\
KEGGdir & 36620  & 20 & 0   & 19 &      55.818182 &  11.745792 &   1450.272727 &  253.433262 &   1.000000 &  0.000000 &                                   -0.950094 &  0.008742 &                                  -0.945906 &  0.034288 &            0.086087 &  0.004343 \\
       &        &    & 500 & 19 &      42.909091 &   0.700649 &    173.727273 &    2.901410 &   8.352503 &  1.470270 &                                   -1.004278 &  0.009252 &                                  -0.948952 &  0.031248 &            0.086368 &  0.004159 \\
3DRoad & 326155 & 3  & 0   & 9  &      68.000000 &            &  82200.000000 &             &   1.000000 &           &                                    0.773300 &           &                                   1.436000 &           &            0.298200 &           \\
       &        &    & 500 & 9  &      19.000000 &            &   7306.000000 &             &  11.251027 &           &                                    0.128400 &           &                                   1.169000 &           &            0.126500 &           \\
\bottomrule
\end{tabular}

\end{sidewaystable}

\stopcontents[supplementary]

\end{document}